\documentclass{article}

\usepackage[english]{babel}

\usepackage[letterpaper,top=1in,bottom=1in,left=1in,right=1in,marginparwidth=1.75cm]{geometry}

\usepackage{amsmath,amsfonts,mathtools,amsthm,amssymb}
\usepackage[most]{tcolorbox}
\usepackage{graphicx}
\usepackage{subcaption}
\usepackage{algorithm}
\usepackage{algorithmic}
\usepackage{lscape}
\usepackage{booktabs}
\usepackage{authblk}
\usepackage{geometry}
\usepackage{pdflscape}
\usepackage{makecell}
\usepackage{utfsym}
\usepackage{todonotes}
\usepackage{epstopdf}

\usepackage{hyperref}
\usepackage{caption}

\DeclareMathOperator*{\argmin}{arg\,min}
\newcommand{\de}{\,\text{d}}
\newcommand{\R}{\mathbb{R}}

\newcommand{\Ex}{\mathbb{E}}
\newcommand{\wprox}{\text{WProx}}

\newcommand\keywordsname{Key words}
\newcommand\AMSname{AMS subject classifications}

\newtheorem{proposition}{Proposition}[section]

\newtheorem{remark}{Remark}[section]

\newcommand{\bGamma}{\boldsymbol{\Gamma}}

\makeatletter
\def\blfootnote{\gdef\@thefnmark{}\@footnotetext}
\makeatother

\newenvironment{@abssec}[1]{%
     \if@twocolumn
       \section*{#1}%
     \else
       \vspace{.05in}\footnotesize
       \parindent .0in
         {\upshape\bfseries #1. }\ignorespaces
     \fi}
     {\if@twocolumn\else\par\vspace{.1in}\fi}
\newenvironment{keywords}{\begin{@abssec}{\keywordsname}}{\end{@abssec}}

\title{{\textbf{ Wasserstein proximal operators describe score-based generative models and resolve memorization } }}

\author[1]{Benjamin J. Zhang}
\author[2]{Siting Liu}
\author[3]{Wuchen Li}
\author[1]{Markos A. Katsoulakis}
\author[2]{Stanley J. Osher}
\affil[1]{Department of Mathematics and Statistics, University of Massachusetts Amherst}
\affil[2]{Department of Mathematics, University of California, Los Angeles}
\affil[3]{Department of Mathematics, University of South Carolina, Columbia}

\begin{document}
\maketitle

\begin{abstract}
We focus on the fundamental mathematical structure of score-based generative models (SGMs). We first formulate SGMs in terms of the Wasserstein proximal operator (WPO) and demonstrate that, via mean-field games (MFGs), the WPO formulation reveals mathematical structure that describes the inductive bias of diffusion and score-based models.  In particular, MFGs yield optimality conditions in the form of a pair of coupled partial differential equations: a forward-controlled Fokker-Planck (FP) equation, and a backward Hamilton-Jacobi-Bellman (HJB) equation. Via a Cole-Hopf transformation and taking advantage of the fact that the cross-entropy can be related to a linear functional of the density, we show that the HJB equation is an uncontrolled FP equation. Second, with the mathematical structure at hand, we present an \emph{interpretable kernel-based} model for the score function which dramatically improves the performance of SGMs in terms of training samples and training time. In addition, the WPO-informed kernel model is explicitly constructed to avoid the recently studied memorization effects of score-based generative models. The mathematical form of the new kernel-based models in combination with the use of the terminal condition of the MFG reveals new explanations for the manifold learning and generalization properties of SGMs, and provides a resolution to their memorization effects. Finally, our mathematically informed, interpretable kernel-based model suggests new scalable bespoke neural network architectures for high-dimensional applications.  
\end{abstract}

\blfootnote{
M. Katsoulakis and B. Zhang are partially funded by AFOSR grant FA9550-21-1-0354. M.K. is partially funded by  NSF DMS-2307115 and  NSF TRIPODS CISE-1934846. 
S. Liu and S. Osher are partially funded by AFOSR MURI FA9550-18-502 and ONR N00014-20-1-2787. W. Li is supported by AFOSR MURI FP 9550-18-1-502, AFOSR YIP award No. FA9550-23-1-0087, NSF DMS-2245097, and NSF RTG: 2038080.}

\normalsize
\section{Introduction}

Proximal operators are powerful tools in optimization \cite{parikh2014proximal}. \emph{Wasserstein} proximal operators (WPO) have been found to be crucial tools for formulating flow-based generative models through connections made by  Hamilton-Jacobi equations \cite{osher2023hamilton} and mean-field games \cite{zhang2023mean,li2023kernel}. This paper formulates score-based generative models (SGMs) \cite{song2020score} and related denoising diffusion models \cite{ho2020denoising} with the Wasserstein proximal operator. We show that SGMs can be concisely described in terms of the Wasserstein proximal operator of \emph{cross-entropy}, which places the method in the context of other generative flows and provides new pathways for analysis of SGMs. For example, the manifold learning aspects of proximal operators on 1-Wasserstein space have been noted in \cite{birrell2022f, gu2023lipschitzregularized}; the Wasserstein proximal operator therefore provides an additional justification for the manifold learning properties of SGMs \cite{pidstrigach2022score}.   

The empirical success of SGMs has induced much research into their statistical properties and mathematical foundations \cite{pidstrigach2022score,de2022convergence,kwon2022score}. There is much interest in how to construct better models for SGMs that train faster with less data. {Some examples include designing critically damped diffusions that converge to Gaussians faster \cite{dockhorn2021score}, finding the optimal time horizon \cite{franzese2023much}, implementing high order SDE integrators \cite{zhang2022fast}, or distilling diffusion models into consistency models \cite{song2023consistency}.} {More interestingly, there is a growing body of work documenting so-called \emph{memorization effects} of diffusion models \cite{li2024good,somepalli2023diffusion,somepalli2023understanding,gu2023memorization,carlini2023extracting}. In particular, \cite{li2024good} presents an example where learning a score function well with the denoising score-matching objective produces a generative model that is equivalent to a kernel density estimate. Moreover, they demonstrated that the primary generative capabilities of the resulting model were due to early stopping, which produced samples equivalent to noisy replicas of the training data. Furthermore, this memorization effect has even been noted in major newspapers \cite{nytimes2024}, in which they demonstrated that a popular score-based text-to-image generator, Midjourney, may output mildly modified copyrighted images that were likely part of its training data. Copyright infringement is a broader, systemic challenge for generative modeling that may entangle engineers and practitioners in litigation \cite{nytimes2023}. }

The relationship between Wasserstein proximal operators and generative modeling has previously been explored in gradient flows \cite{salim2020wasserstein} and in generative adversarial networks \cite{lin2021wasserstein}. We present a fundamental characterization of SGMs in terms of the Wasserstein proximal operator which yields an \emph{interpretable} kernel-based approach to approximating score functions. In particular, via kernel formulas for computing regularized Wasserstein proximal operators \cite{li2023kernel} and their relation to Hamilton-Jacobi-Bellman (HJB) equations, we construct a \emph{WPO-informed} kernel model for the score function using the Green's functions related to the HJB equation. Moreover, we demonstrate that enforcing the terminal condition of the associated HJB equation to learn the parameters of the kernel-based formula is tantamount to learning the manifold on which the data lies. {The combination of the mathematical form of the WPO-kernel model and its training procedure provides a resolution to the memorization effects of SGMs.} Empirical evidence demonstrate that the mathematical structure encoded by WPO-informed kernel model exhibits faster learning of the score function with less data {and produces a SGM that generalizes}.  We discuss our contributions in the following two sections.

\paragraph{Contribution 1: Formulating score-based generative models through Wasserstein proximal operators.}

In Section~\ref{sec:wposgm}, we provide a fundamental characterization of score-based generative model as Wasserstein proximal operators of cross-entropy. The mean-field games (MFG) formulation of SGMs and regularized WPOs established in \cite{zhang2023mean} and \cite{li2023kernel}, respectively, establish the connection between SGMs and WPOs. The optimality conditions of the SGM MFG yield a pair of PDEs: a \emph{forward} controlled Fokker-Planck (FP) equation, and a \emph{backward} Hamilton-Jacobi-Bellman (HJB) equation. Moreover it was shown in \cite{zhang2023mean} that the implicit score-matching objective is a direct consequence of the MFG formulation, and shows that the ISM objective is related to simultaneously minimizing the Wasserstein metric and the cross-entropy loss.

Our primary contribution in this section is connecting the Wasserstein proximal operator of linear energies \cite{li2023kernel} to the MFG perspective of SGMs \cite{zhang2023mean}. Kernel formulas for approximating the Wasserstein proximal operator in \cite{li2023kernel}, and this connection allows for the construction of kernel-based models for the score function. We will review how the particular form of the MFG, specifically the inclusion of the cross-entropy as the terminal cost decouples the set of PDEs, implying that the HJB equation alone characterizes SGM. Moreover, by a Cole-Hopf transformation, the HJB equation is equivalent to an uncontrolled FP equation \cite{evans2022partial,zhang2023mean}. The MFG's \emph{forward-backward} structure explains the noising-denoising nature of score-based generative models. Establishing the connection between WPOs and SGMs justifies the use of kernel formulas from \cite{li2023kernel} to provide a closed form solution to the score function, which is the topic of the next section. 

We also highlight an observation that the backward-forward PDE system which arises from the SGM MFG provides an alternative perspective for reversing the heat equation. Solving the reverse heat equation is an ill-posed problem in numerical PDEs. By applying PDE theory (mean-field games) to the SGM system, we discuss how the one-way coupled PDE system corresponding to SGMs yields a well-posed system with an associated optimization problem that has the effect of reversing the heat equation, provided that some minimal information about the initial condition is available.

\paragraph{Contribution 2: Deconstructing score-based models with an interpretable WPO-informed kernel model that \emph{resolves memorization and generalizes}.}

With the WPO formulation of SGM at hand, our main contribution in Section \ref{sec:wpokernelmodel} is the WPO-informed kernel model which is an \emph{interpretable} model that unravels the mathematical structure of the score function. The WPO formulation of SGM provides a description of the \emph{inductive bias} of SGM, that is, the inherent mathematical structure of the score function instantiated as the gradient of the solution of the HJB equation \eqref{eq:hjb}. The solution of the HJB equation can be written in closed form as the Cole-Hopf transform of a kernel model. Incorporating the structure of the HJB equation allows the WPO-informed kernel score model to be trained quickly with less data. In Section \ref{sec:numerical}, we demonstrate the advantages of the kernel formula in illustrative numerical examples. The kernel-based model in \eqref{eq:kernelcovdensity} is the main result of the WPO formulation of SGMs, which we preview here
\begin{align*}
    \hat{\pi}_\theta(x; \{Z_i\}_{i = 1}^N) = \frac{1}{N} \sum_{i = 1}^N \frac{\det \bGamma_\theta(Z_i)}{(2\pi)^{d/2}}\exp\left( -\frac{(x-Z_i)^\top\bGamma_\theta(Z_i) (x-Z_i)}{2} \right).
\end{align*}
The kernel in the WPO-informed model is informed by the \emph{Green's function} of the HJB equation, which comes from the optimality conditions of the MFG. The WPO-informed kernel model is a linear combination of Green's function centered on the training data, which solves the HJB equation exactly by construction. Given training data $\{Z_i\}_{i = 1}^{N_\text{train}}$ from data distribution $\pi$ we learn the precision $\bGamma_\theta(Z_i)$ (inverse covariance matrices) around each of the kernel centers, which are a subset of the training data\footnote{In numerical experiments, we chose a subset of training data to be the kernel centers, but in general this is not necessary.}. It is simply an extension of the kernel formulas for approximating the Wasserstein proximal operators of linear operators first derived in \cite{li2023kernel}. Ostensibly this model is simply a Gaussian mixture model. {The view that SGMs are GMMs or kernel formulas has been noted recently in \cite{li2024good}, which further supports the use of the WPO-kernel model.} The WPO formulation of SGM \eqref{eq:score_opt} and its solution described by Proposition \ref{prop:repformula} provides a kernel representation formula that encodes the inductive bias of SGMs. A Gaussian mixture model is a natural choice for expressing the inductive bias as approximating the kernel representation formula \eqref{eq:repformula} admits a kernel score model that can be computed in closed form. This fact is presented in Proposition \ref{prop:hjbsolution}.

A key difference from a standard GMM is that the model parameters $\theta$ are trained via implicit score-matching only at the terminal condition. As the kernel-based formula solves the HJB equation exactly except at the terminal condition, there is no need to perform score-matching for any other time other than the terminal time (see Proposition \ref{prop:hjbsolution}).
The local precision matrices are learned by minimizing the implicit score-matching objective, which has the effective of learning the underlying Riemannian manifold on which the data distribution lies \cite{roweis2000nonlinear}. {This allows the resulting generative model to generalize and avoid memorization effects}.

The resulting model provides an explainable, interpretable formulation of score-based generative model grounded through interconnections among MFGs, information theory, optimal transport, manifold learning, and optimization. Our model clarifies the effectiveness of \emph{early stopping} and manifold learning properties of SGMs. Moreover, the manifold learning properties may mathematically explain the benefits of latent diffusion methods \cite{vahdat2021latentSGM,rombach2021latentdiffusion}, where manifold learning is treated separately from the SGM.
\begin{figure}[h]
    \centering
    \includegraphics[width = 0.6\textwidth, trim={0 200 0 0},clip]{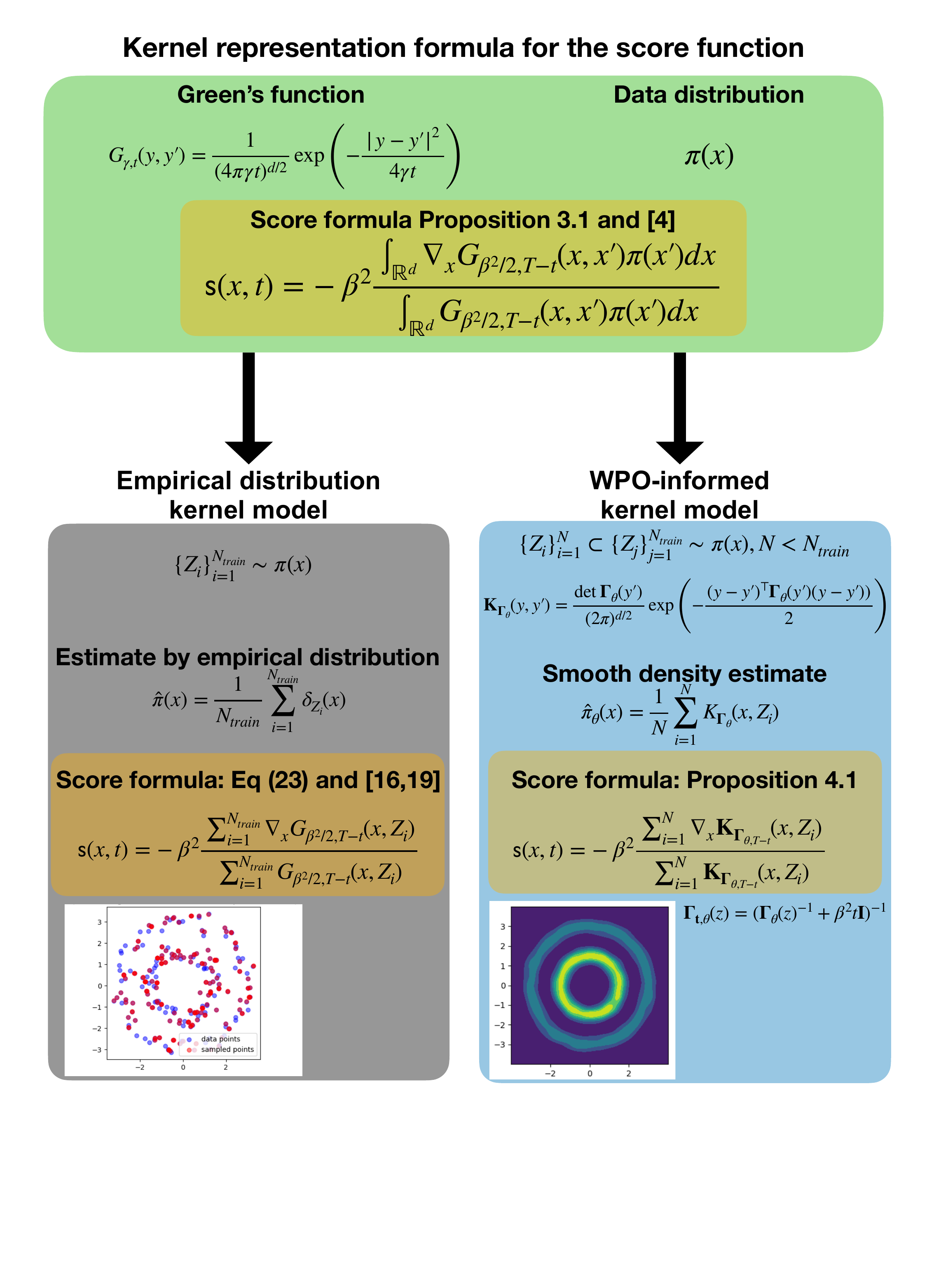}
    \caption{ The core idea of this paper is that since the score function has a kernel representation formula, approximations to the score should respect the structure of the formula. Use of the empirical distribution to construct a kernel formula for the score function memorizes the training data. Our WPO-informed kernel model learns local precision matrices via the terminal condition of a HJB equation, which produces a kernel-based model that generalizes better and exhibits manifold learning. }
    \label{fig:wposgm} 
\end{figure}
\section{Background on score-based generative models and Wasserstein proximal operators}
\label{sec:background}
We review score-based generative models with SDEs and kernel approximation to Wasserstein proximal operators of linear energies as presented in \cite{song2020score} and \cite{li2023kernel}, respectively. In Section \ref{sec:wposgm} we show how the two seemingly disparate topics are mathematically equivalent.

\paragraph{Notation.} Let $\mathcal{P}_2(\R^d)$ denote the space of probability distributions with finite second moments. The pair $(\mathcal{P}_2(\R^d),\mathcal{W}_2)$ is a complete separable metric space on $\mathcal{P}_2(\R^d)$, where $\mathcal{W}_2(\rho,\rho')$ is the $2$-Wasserstein distance \cite{santambrogio2015optimal}. The density functions of probability distributions $\rho\in \mathcal{P}_2(\R^d)$ are denoted $\rho(x)$. 

\subsection{Score-based generative models} \label{subsec:sgm}
Suppose we wish to produce samples from target data distribution $\pi \in \mathcal{P}_2(\R^d)$ given only a finite number of samples $\{X_i\}_{i = 1}^N\sim \pi$. Let $f:\R^d\times [0,T]\to \R^d$ be a vector field for some scalar $T>0$, and let $\beta(t)$ be a positive function for $t\in [0,T]$. Let $Y(s)$ be the stochastic process evolving according to SDE
\begin{align}\label{eq:noising}
    \de Y(s) &= -f(Y(s),T-s) \de s + \beta(T-s) \de W(s) \\
    Y(0) &\sim \pi \nonumber
\end{align}
where $Y(s) \sim \eta(\cdot,s)$ and $W(s)$ is a standard Brownian motion. This process adds noise to samples from $\pi$; typically $f$ and $\beta$ are chosen so that $\pi$ is approximately normal for sufficiently large $T$. Score-based generative models aim to learn a controlled SDE that reverses the evolution of the noising process. In \cite{anderson1982reverse} it was first found that the process $X(t)\sim\rho(\cdot,t)$ 
\begin{align}\label{eq:denoising}
    \de X(t) & = \left[ f(X(t),t) + \beta(t)^2 \nabla \log \eta(X(t),T-t)\right] \de t + \beta(t) \de W(t) \\
    X(0) &\sim \eta(\cdot,T) \nonumber
\end{align}
reverses the evolution of $Y(s)$ with $\rho(\cdot,t) = \eta(\cdot,T-t)$ so that $\rho(\cdot,0) = \eta(\cdot,T)$ and $\rho(\cdot,T) = \pi(\cdot)$. Assuming the score function $\mathsf{s}(y,s) = \nabla \log \eta(y,s)$ is learned, then new samples from $\pi$ can be obtained by evolving samples from $\eta(\cdot,T)$ (which are approximately normal) through the SDE of $X(t)$. To learn the score function, \cite{song2020score} trains a neural net $\mathsf{s}_\theta:\R^d \times[0,T] \to \R^d$ by minimizing a score-matching loss function. The explicit score-matching objective (ESM) 
\begin{align}\label{eq:esm}
        C_{ESM}(\theta) =\int_0^T  \beta(T-s)^2\Ex_{\eta(\cdot,s)}\left[  \| \mathsf{s}_\theta(Y(s),s) - \nabla \log \eta(Y(s),s) \|^2 \right] \de s 
\end{align}
is generally intractable to compute since evaluations of the density function are not available \cite{song2020sliced}. One of the practical alternatives is the implicit score-matching objective (ISM) \cite{song2020sliced}
\begin{align}\label{eq:ism}
     C_{ISM}(\theta) =  \int_0^T  \beta(T-s)^2\Ex_{\eta(\cdot,s)}\left[  \|\mathsf{s}_\theta(y,s) \|^2 + 2\nabla \cdot \mathsf{s}_\theta(y,s) \right] \de s
\end{align}
which applies integration-by-parts to avoid direct evaluations of the score function. Here $\nabla \cdot$ denotes the divergence operator. Another loss function is the denoising score-matching (DSM) objective \cite{song2020score,ho2020denoising,vincent2011connection}
\begin{align}\label{eq:dsm}
     C_{DSM}(\theta) = \int_0^T  \beta(T-s)^2 \Ex_{Y_0 \sim \pi} \Ex_{Y(s) \sim \eta(\cdot,s|Y_0)}\left[\|\mathsf{s}_\theta(Y(s),s) - \nabla \log \eta(\cdot,s|Y_0) \|^2 \right]\de s,
\end{align}
 which requires knowledge of the conditional score function $\nabla \log \eta(y,s|Y_0)$. The functions $f$ and $\beta$ are typically chosen so that the conditional score function is known in closed form. The DSM objective is most frequently used in practical applications as it does not require gradient evaluations when the noising process is a linear SDE. While these three objective functions are equivalent in the limit of infinite data, i.e., they share the same minimizers, the properties of their empirical losses are not fully understood.

 \subsection{Kernel formulas for regularized Wasserstein proximal operators} %

 Proximal operators often arise in the optimization of nonconvex and nonsmooth functions \cite{parikh2014proximal}. The Wasserstein proximal operator (WPO) is the instantiation of the proximal operator on the space of probability distributions with finite second moments $\mathcal{P}_2(\R^d)$. The WPO is a mapping on the Wasserstein space which is a metric space on $\mathcal{P}_2(\R^d)$ along with the $2$-Wasserstein distance \cite{santambrogio2015optimal}. Intuitively, given some functional on Wasserstein space (e.g., a probability divergence with respect to some fixed distribution), one evaluation of the proximal operator can be interpreted as a single step of gradient descent for that functional on Wasserstein space.

Given a potential energy function $V:\R^d\to \R$, a linear functional on Wasserstein space is
\begin{align}
    \mathcal{V}(\rho) = \int V(x) \rho(x) \de x.
\end{align}
Often, the potential energy function is the negative log density of some target density, i.e., $V(x) = -\log \pi(x)$. The Wasserstein proximal operator of $\mathcal{V}$ is a mapping on $\mathcal{P}_2(\R^d)$
\begin{align}\label{eq:wpo}
    \rho_h = \wprox_{h\mathcal{V}}(\rho_0) \coloneqq  \argmin_{\rho \in \mathcal{P}_2(\R^d)} \mathcal{V}(\rho) + \frac{\mathcal{W}_2(\rho_0,\rho)^2}{2h},
\end{align}
where $h>0$ is a scalar constant. The WPO has a dynamic formulation via the Benamou-Brenier formula \cite{santambrogio2015optimal}: for any $T>0$,
\begin{align}
    \frac{\mathcal{W}_2^2(\rho_0,\mu)}{2h} = \inf_{v}\left\{ \int_0^h \int_{\R^d} \frac{1}{2}|v(x,t)|^2 \rho(x,t) \de x \de t : \partial_t \rho + \nabla \cdot(v\rho) = 0, \rho(x,0) = \rho_0(x), \rho(x,h) = \mu(x) \right\}.
\end{align}
The WPO is, therefore, associated with the potential mean-field game
\begin{align} \label{eq:wpomfg}
    &\min_{\rho,v} \left\{- \int_{\R^d} V(x) \rho(x,h) \de x + \int_0^T\int_{\R^d} \frac{1}{2}|v(x,t)|^2 \rho(x,t) \de x \de t \right\}\\
    \text{s.t. } & \frac{\partial \rho}{\partial t} + \nabla \cdot(v\rho) = 0, \, \rho(x,0) = \rho_0(x) \nonumber,
\end{align}
which has optimality conditions in the form of two coupled nonlinear partial differential equations (PDEs)
\begin{align}
    \begin{dcases}
        -\frac{\partial U}{\partial t} + \frac{1}{2}\left|\nabla U\right|^2 = 0,\, U(x,h) = V(x) \\
        \frac{\partial \rho}{\partial t} - \nabla \cdot(\rho \nabla U) = 0,\, \rho(x,0) = \rho_0(x).
    \end{dcases}
\end{align}
The first equation is a Hamilton-Jacobi equation, while the second is simply the continuity equation with $v = -\nabla U$. The output of the WPO operator applied to the input distribution $\rho_0$ is the density function $\rho(x,T)$. Solving such a system in general requires numerical approximation; however, standard strategies for solving PDEs are quite intractable in high dimensions. To address the curse of dimensionality, \cite{li2023kernel} proposed the \emph{regularized} Wasserstein proximal operator by introducing second derivative terms in the optimality conditions, which can be interpreted as a form of \emph{entropic} regularization, similar to the entropic formulation of optimal transport \cite{chen2016relation}. Doing so allows using kernel formulas which provides a closed form solution to the resulting mean-field game. 
The regularized Wasserstein proximal operator with entropic regularization parameter $\gamma>0$, $\rho_T(x) = {\wprox_{T\mathcal{V}}}_{,\gamma}(\rho_0)$, is associated with the PDE system 
\begin{align} \label{eq:rwpo_mfg}
    \begin{dcases}
        -\frac{\partial U}{\partial t} + \frac{1}{2}\left|\nabla U\right|^2 = \gamma\Delta U,\, U(x,T) =V(x) \\
        \frac{\partial \rho}{\partial t} - \nabla \cdot(\rho\nabla U) = \gamma \Delta \rho,\, \rho(x,0) = \rho_0(x),
    \end{dcases}
\end{align}
where $\rho_T = \rho(x,T)$. In \cite{li2023kernel}, it was shown that via a Cole-Hopf transform, the HJB equation in \eqref{eq:rwpo_mfg} is equivalent to the heat equation, which can be solved via Green's functions.  Recall that the \emph{heat kernel}, $G_{\gamma,t}(y,y')$, for diffusion constant $\gamma$ is the \emph{Green's function} of the heat equation and the convolution operator $*$ for a function $\phi:\R^d\to \R$ are defined 
\begin{align}
    G_{\gamma,t}(y,y') = \frac{1}{(4\pi\gamma t)^{d/2}} \exp\left(- \frac{|y-y'|^2}{4\gamma t} \right), \; (G_{\gamma,t}*\phi)(x) = \int_{\R^d} G_{\gamma,t}(x,x') \phi(x') \de x'.
\end{align}
The solution to the PDE system \eqref{eq:rwpo_mfg} can therefore be expressed as
\begin{align}\label{eq:wpokernelintegral}
U(x,t) = -2\gamma \log \left(G_{\gamma,T-t} * e^{-\frac{V(x)}{2\gamma}} \right).
\end{align}

\section{Deriving score-based generative models from Wasserstein proximal operators}
\label{sec:wposgm}

In this section we show how score-based generative models can be formulated in terms of a regularized Wasserstein proximal operator of cross-entropy. Minimizing the cross-entropy functional over spaces of probability distributions is a frequent task in machine learning \cite{cover1999elements,rubinstein2004cross}. We will show that score-based generative models produces samples of a single application of the Wasserstein proximal of the cross-entropy functional. Using results from \cite{zhang2023mean}, we derive canonical formulations score-based generative models as presented in \cite{song2020score}.

\subsection{Deriving SGM from regularized proximal operators of cross-entropy}\label{subsec:rwpo_to_sgm}

\label{subsec:regWPO}
Let $\pi \in \mathcal{P}_2(\R^d)$ be the target data distribution as defined in Section \ref{subsec:sgm} and let $\mu\in \mathcal{P}_2(\R^d)$ be some arbitrary distribution. The cross-entropy of a distribution $\pi$ with respect to $\mu$ is defined
\begin{align}
    \mathcal{H}(\mu,\pi) \coloneqq -\Ex_{\mu}\left[ \log \pi(x)\right] = -\int_{\R^d} \mu(x) \log \pi(x) \de x.
\end{align}
The cross-entropy method is related to the Kullback-Leibler divergence and the entropy $\mathcal{E}(\mu) = -\Ex_\mu[\log \mu(x)]$ by the formula $\mathcal{D}_{KL}(\mu\|\pi) = -\mathcal{E}(\mu) + \mathcal{H}(\mu,\pi)$. Minimizing the cross-entropy appears frequently in numerous machine learning tasks \cite{cover1999elements,rubinstein2004cross}; we will show that it also plays an interesting role in score-based generative models.

Starting from the regularized WPO mean-field game in \eqref{eq:rwpo_mfg} and results from \cite{zhang2023mean}, we show that SGMs can be fundamentally understood in terms of the cross-entropy loss and Wasserstein proximal operator. In particular, through a judicious choice of the regularization parameters $\gamma$ and $h$, we will show that \eqref{eq:rwpo_mfg} is equivalent to the SGM formulation in Section \ref{subsec:sgm}. First, consider the Wasserstein proximal operator in \eqref{eq:wpo} for $h = \beta^2T$ for $\beta,T>0$, and note its equivalence to 
\begin{align}\label{eq:wpo}
    \wprox_{h\mathcal{H}}(\rho_0) = \argmin_{\rho\in \mathcal{P}_2(\R^d)} \beta^2\mathcal{H}(\rho,\pi) + \frac{\mathcal{W}_2(\rho_0,\rho)^2}{2T}.
\end{align}
Next, by choosing the entropic regularization parameter $\gamma = \beta^2/2$, notice that we obtain the mean-field game optimality conditions for the regularized WPO, $\wprox_{\beta^2T\mathcal{H},\beta^2/2}(\rho_0)$,
\begin{align}\label{eq:score_opt}
\begin{dcases}
    -\frac{\partial U}{\partial t}+ \frac{1}{2}|\nabla U|^2 = \frac{\beta^2}{2}\Delta U,\quad U(x,T) = -\beta^2 \log \pi(x) \\
    \frac{\partial \rho}{\partial t}- \nabla \cdot(\rho \nabla U) = \frac{\beta^2}{2}\Delta \rho,\quad \rho(x,0) = \rho_0(x)    \end{dcases}
\end{align}
Combining the Cole-Hopf transformation with a time reparametrization, we apply the variable transformation $U(x,t) = -\beta^2 \log \eta(x,T-t)$ to the Hamilton-Jacobi equation and obtain the system 
\begin{align}
\begin{dcases}\label{eq:fpfp}
    \frac{\partial \eta}{\partial s} = \frac{\beta^2}{2}\Delta \eta,\, \eta(y,0) = \pi(y) \\
    \frac{\partial \rho}{\partial t}+ \nabla \cdot(\rho \beta^2  \nabla \log \eta) = \frac{\beta^2}{2}\Delta \rho,\quad \rho(x,0) = \rho_0(x) 
    \end{dcases}
\end{align}
The first equation is the heat equation, which in terms of SDEs, corresponds with an uncontrolled Brownian motion with diffusion coefficient $\beta$. The second equation is a controlled Fokker-Planck equation, where the controller is determined by the score function of the uncontrolled Brownian motion. Notice that these are precisely the noising \eqref{eq:noising} and denoising systems \eqref{eq:denoising} described in Section \ref{subsec:sgm} for $f \coloneqq 0$. Furthermore, if $\rho_0(x)$ is chosen to be the density $\eta(\cdot,T)$ at $s = T$, then (\cite{zhang2023mean}, Theorem 4.1) showed that the controlled Fokker-Planck will reverse the evolution of the heat equation. To solve the system of PDEs, one needs to find the solution to the HJB equation, whose gradient is precisely the score function. Therefore, we have shown that the pair of noising and denoising SDEs that appear in score-based generative models are naturally encoded in the backward-forward structure of the mean-field game representation of the regularized Wasserstein proximal operator.

Moreover, in (\cite{zhang2023mean}, Theorem 4.2), it was shown that the mean-field game in \eqref{eq:wpomfg} can be directly related to the implicit score matching objective in \eqref{eq:ism}, meaning that optimizing the ISM and learning the score is equivalent to solving the MFG \eqref{eq:score_opt}. With these connections, score-based generative models can be summarized with the forward and inverse $\wprox$ notation:
\begin{align}
\label{eq:proximal_back_fwd}
    \pi\approx \tilde{\pi} = \wprox_{\beta^2T \mathcal{H},\beta^2/2}( \wprox_{\beta^2T \mathcal{H},\beta^2/2}^{-1}(\hat{\pi})),
\end{align}
where $\hat{\pi}$ is the empirical distribution defined by samples $\{X_i\}_{i = 1}^N \sim \pi$, {and $\tilde{\pi}$ is the generated distribution with an approximate score function.} Here, the inverse $\wprox$ should be understood as a mapping, i.e., the set of densities $\rho$ such that $\wprox_{\beta^2T\mathcal{H},\beta^2/2}(\rho) = \hat{\pi}$. 

The connection between Wasserstein proximal operator and score-based generative modeling allows the use of kernel formulas to represent the score function. In particular, by appealing to kernel representation formula of the solution to the HJB equation \eqref{eq:wpokernelintegral}, we can express the score function in terms of a similar kernel representation formula. 
\begin{proposition}{(Kernel representation formula for the score function).}\label{prop:repformula}
For initial condition $\eta(x,0) = \pi(x)$, the score function $\mathsf{s}(x,t) = \nabla \log \eta(x,t)$ in the denoising SDE \eqref{eq:denoising} has the kernel representation formula 
    \begin{align} 
        \mathsf{s}(x,t) &= -\nabla U(x,t) = \beta^2 \nabla \log \left((G_{\beta^2/2,{T-t}} * \pi)(x) \right) = \beta^2 \frac{(\nabla G_{\beta^2/2,T-t} * \pi)(x)}{(G_{\beta^2/2,T-t} * \pi)(x)}\\
       \mathsf{s}(x,t)  &= -\beta^2 \frac{\int_{\R^d} \frac{x-x'}{\beta^2(T-t)} G_{\beta^2/2,{T-t}}(x,x') \pi(x') \de x' }{\int_{\R^d} G_{\beta^2/2,T-t}(x,x') \pi(x') \de x'}.\label{eq:repformula}
    \end{align}
\end{proposition}
\begin{proof}
    The score function in \eqref{eq:denoising} is the gradient of the logarithm of the solution to \eqref{eq:fpfp}. The solution to \eqref{eq:fpfp} is related to the solution of \eqref{eq:score_opt} by a Cole-Hopf transform. The solution to \eqref{eq:score_opt} has a kernel representation formula as stated in Proposition 4 of \cite{li2023kernel} and \eqref{eq:wpokernelintegral}. Taking the gradient of the kernel \eqref{eq:wpokernelintegral} gives us the representation formula for the score function. 
\end{proof}

This kernel representation formula for the score function concisely describes the inductive bias of the score function. Proposition \ref{prop:repformula} implies that given parameters $t$ and $\beta$, the Green's function of the associated PDE determines the kernel of the representation formula. Therefore, the problem at hand is to find a suitable approximation of the  \emph{true, unknown} density $\pi(x)$ to approximate the score function. In Section \ref{sec:wpokernelmodel} we suggest two suitable approximations for the density function. {Another kernel approximation to the score function using interacting particle systems was studied the context of approximating solution to the Fokker-Planck equation of the Langevin dynamics \cite{maoutsa2020interacting}. }

\begin{remark}\textbf{(Entropic regularization parameter defines equivalence classes)} For a given $h>0$ in the Wasserstein proximal operator, there is a corresponding equivalence class of score-based models with diffusion coefficient $\beta$ and time horizon $T$ such that $\beta^2 T = h$. This perspective directly links the chosen parameters in SGMs to the step parameter in the Wasserstein proximal operator. The parameter $h = \beta^2T$ can therefore be interpreted as a diffusion scaling. 
\end{remark}

\begin{remark}\textbf{(Why is the Wasserstein proximal necessary?)}
    Recent state-of-the-art generative models are fundamentally described in terms of Wasserstein proximal formulations. For example, the optimal transport flow (OT-flow) \cite{onken2021ot} is the 2-Wasserstein proximal of the KL-divergence, and the $(f,\Gamma)$ generative adversarial network is the 1-Wasserstein proximal of $f$-divergences \cite{birrell2022f}. By placing score-based and denoising diffusion generative models in terms of the 2-Wasserstein proximal, we make connections to other generative models and the power of proximal Wasserstein operators more broadly. In particular, recently the {manifold learning} property of proximal operators have been noted on $1$-Wasserstein space \cite{birrell2022f,gu2023lipschitzregularized}. Placing SGMs in terms of the WPO may provide new explanations for its manifold learning properties \cite{pidstrigach2022score,de2022convergence}. 
\end{remark}

\begin{remark}\textbf{(SGMs reverse the heat equation with MFG theory and initial data)}
Given the heat equation $\partial_s \eta = \frac{\beta^2}{2} \Delta \eta$ and terminal condition $\eta(y,T)$, the task of solving for $\eta(y,s)$ for $s \in [0,T)$ is known as reversing the heat equation. Traditional discretization schemes are known to be numerically unstable as errors will accumulate and amplify \cite{leveque2007finite}. Ostensibly, score-based generative models appear to provide a numerically stable method for reversing the heat equation; we, however, emphasize there is a subtle but significant difference between solving the backward heat equation and SGM. In contrast to reversing the heat equation, SGM also has \emph{partial knowledge} of the initial distribution $\pi(y)$ in the form of samples. Having partial knowledge is empirically sufficient to make the problem better conditioned. This fact can be understood rigorously through MFG theory. The one-way coupled 
PDE systems \eqref{eq:score_opt} and \eqref{eq:fpfp} that arise as optimality conditions of the MFG formulation of SGMs (Theorem 7, \cite{zhang2023mean}) and the equivalence of score-matching to the MFG system (Theorem 10, \cite{zhang2023mean}) show that SGMs are well-posed. In particular, the Hamilton-Jacobi equation in \eqref{eq:score_opt} is a second order equation, and so standard regularity theory for viscous HJB equations imply that the PDE is well-posed \cite{evans2022partial}. Moreover, assuming regularity of the terminal condition, solutions to second order HJB equations are classical, as they are unable to form shocks (discontinuities). Therefore, the gradient of the HJB solution exists in a classical sense and the corresponding controlled Fokker-Planck that corresponds with the denoising SDE is well-posed too. MFG theory reveals that reversing the heat equation is a well-posed problem provided that the score function is known or can be learned. 
\end{remark}

\section{Deconstructing score-based generative models: a new WPO-informed kernel model}
\label{sec:wpokernelmodel}

We have shown that score-based generative models are regularized Wasserstein proximal operators of cross-entropy. In this section we revisit the equivalent mean-field games formulation of score-based generative models, and use it to construct a new WPO-informed kernel model for the score function. Mean-field games in general are quite difficult to solve because of their nonlinear, coupled nature. For the Wasserstein proximal of cross-entropy, they are intriguingly decoupled, and reduce to a control problem. In particular, the Hamilton-Jacobi-Bellman equation together with its terminal condition,
\begin{align}
\begin{dcases}\label{eq:hjb}
     -\frac{\partial U}{\partial t}+ \frac{1}{2}|\nabla U|^2 = \frac{\beta^2}{2}\Delta U \\
     U(x,T) = -\log \pi(x),
     \end{dcases}
\end{align}
is sufficient to characterize the solution of the mean-field game. Using the HJB equation, we will build a kernel-based model for the score function. 

\subsection{The empirical kernel model for the score function memorizes the training data}\label{sec:simplekernel}
The equivalence of the HJB equation in \eqref{eq:hjb} to a Fokker-Planck equation in \eqref{eq:fpfp} via a Cole-Hopf transformation was first used to construct kernel formulas for regularized Wasserstein operators of linear energies in \cite{li2023kernel}. From a probabilistic perspective, this is equivalent to the use of the probability transition kernel of the FP equation's corresponding stochastic process to compute the density function at future times. Given samples $\{Z_i\}_{i = 1}^{N_\text{train}}$ from distribution $\pi$, the empirical distribution 
\begin{align}
    \pi(\cdot) \approx \hat{\pi}(\cdot) = \frac{1}{N_\text{train}}\sum_{i = 1}^{N_\text{train}} \delta_{Z_i}(\cdot)  \label{eq:empiricaldistribution}
\end{align}
can serve as an approximation to the true distribution $\pi$. Appealing to the kernel representation formula for the solution to the HJB \eqref{eq:wpokernelintegral} is
\begin{align*}
\hat{U}(x,t) = -\beta^2 \log (G_{\beta^2/2,T-t} * \hat{\pi})(x) = -\beta^2 \log \left( \frac{1}{N_\text{train}} \sum_{i = 1}^{N_\text{train}} G_{\beta^2/2,T-t}(x,Z_i) \right).
\end{align*}

Note that this approximate solution only solves the HJB equation exactly for $t \in [0,T)$ as $\hat{U}(x,t)$ is \emph{ill-defined} at $t = T$ due to the delta functions, so the terminal condition is not satisfied in the classical sense. The corresponding score function is still well-defined for $t \in [0,T)$ (i.e., not including $t = T$). Therefore, applying Proposition \ref{prop:repformula}, we have the score formula
\begin{align}\label{eq:overfitsgm}
    \hat{\mathsf{s}}(x,t) = 
    {\color{black}-  \nabla \hat{U}(x, t)=}
    \frac{(\nabla G_{\beta^2/2,T-t} * \hat{\pi})(x)}{(G_{\beta^2/2,T-t} * \hat{\pi})(x)} = -\beta^2 \frac{\sum_{i = 1}^{N_\text{train}} \frac{x-Z_i}{\beta^2(T-t)} G_{\beta^2/2,T-t}(x,Z_i) }{\sum_{i =1 }^{N_\text{train}} G_{\beta^2/2,T-t}(x,Z_i)}.
\end{align}

Using this empirical kernel formula directly, however, produces a score-based generative model that fails to generalize. In fact, when used directly as part of the reverse SDE, this model will \emph{memorize} and \emph{resample} from the training data $\{Z_i\}$. {This behavior is the so-called memorization effect of score-based generative models and has been recently studied in both theoretical and applied contexts \cite{li2023kernel,somepalli2023diffusion,somepalli2023understanding,gu2023memorization}.} In \cite{pidstrigach2022score}, it was shown that as long as the initial condition for the controlled Fokker-Planck equation $\rho_0(x)$ shared the same support as $\eta(y,T)$, and the empirical kernel formula is used as an approximation to the score function, then $\rho(x,T)$ will share the same support as $\hat{\pi}$. Moreover, it was shown in \cite{vincent2011connection} that the denoising score-matching objective \eqref{eq:dsm} is equivalent to explicitly score-matching to the empirical kernel formula. In particular, \cite{pidstrigach2022score} noted that when neural nets are trained with the DSM objective \eqref{eq:dsm}, halting the training early is required to prevent the neural net from fitting exactly to the kernel formula. {A sample complexity argument demonstrating this phenomenon was presented in \cite{li2024good} in which they showed that when a neural net is trained with the DSM objective function, the model fits exactly to the empirical kernel approximation for the score function \eqref{eq:overfitsgm}. Appealing to the Wasserstein proximal notation in \eqref{eq:proximal_back_fwd}, using this empirical score score formula is equivalent applying the relation:
\begin{align}
     \hat{\pi} = \wprox_{\beta^2T \mathcal{H},\beta^2/2}( \wprox_{\beta^2T \mathcal{H},\beta^2/2}^{-1}(\hat{\pi})),
\end{align}} 
{that is, using \eqref{eq:overfitsgm} simply returns the training data. }

Moreover, in both the analysis and implementation of SGMs, early stopping of the denoising process, i.e. simulating \eqref{eq:denoising} up to $t = T-\epsilon$ for some $\epsilon >0$ is assumed to prove convergence of the method, or improved generalization \cite{chen2016relation,de2022convergence,conforti2023score,li2024good}. In continuous time, early stopping is equivalent to sampling from a mollified empirical distribution, i.e. $\eta(x,\epsilon) = (G_\epsilon * \hat{\pi})(x)$ or simply sampling from a Gaussian with variance $\epsilon^2 \mathbf{I}$ around each training point. Early stopping alone may not be enough for effective generalization as $\epsilon$ is not always chosen in an informed manner. One can imagine simple examples where the extent of early stopping is applied is dependent on the data. In Section \ref{sec:num:earlystop} we show numerical examples that show the impact of $\epsilon$ for generalization.

\subsection{A kernel model that generalizes: {resolving memorization}}

An accurate\footnote{Meaning that the score model is able to express the inductive bias of score-based generative models, as defined by the Hamilton-Jacobi-Bellman equation \eqref{eq:hjb}.} score model should solve the PDE \eqref{eq:hjb}, in particular, including the terminal condition. In spite of the memorization effects of the kernel formula, they will be essential for producing a kernel-based SGM that generalizes. Crucially, the Cole-Hopf transformation of the kernel formula is closed under the evolution of the HJB equation. The kernel model, however, fails to match the terminal condition, which is why the empirical kernel model memorizes to the training data. Our approach is to simply alter the kernel formula so that it will be able to better match the terminal condition in \eqref{eq:hjb} while still exactly solving the HJB equations. 

One approach to matching the terminal condition is to mollify the kernel formula at time $t = T$ so that $\hat{\pi}$ is a density function with support that matches the true support of $\pi$ rather than an empirical distribution. Convolving $\hat{\pi}$ with a Gaussian with a fixed bandwidth has the equivalent effect of early stopping with the empirical kernel score model \eqref{eq:overfitsgm}. We will consider a more flexible mollification where the regularizing Gaussian has a state-dependent covariance matrix that is learned from the training data.

\paragraph{A better kernel model.}
Consider the following kernel-based approach that produces a better approximation to $\pi$ in Proposition \ref{prop:repformula}. Like the empirical distribution, Gaussian mixture model approximations to $\pi$ are preferable in light of the representation formula \eqref{eq:repformula} as the integral can be evaluated in closed form. We introduce a matrix-valued function to model the local precision matrix around each of the kernel centers. The inclusion of this precision matrix, in effect, yields a Gaussian mixture model, which is still closed under the evolution of the heat equation. Define a matrix-valued function  $\bGamma_\theta: \R^d \to \R^{d\times d}$ to be a model for the precision matrix. Consider the model 
\begin{align}\label{eq:kernelcovdensity}
    \hat{\pi}_\theta(x; \{Z_i\}_{i = 1}^N) = \frac{1}{N} \sum_{i = 1}^N \frac{\det \bGamma_\theta(Z_i)}{(2\pi)^{d/2}}\exp\left( -\frac{(x-Z_i)^\top\bGamma_\theta(Z_i) (x-Z_i)}{2} \right), 
\end{align}
where $N<N_\text{train}$. This formula is simply a generalization of the empirical distribution by Gaussian kernels where a subset of the training data is used as the kernel centers, and a local covariance matrix is learned around each center. In contrast to the empirical kernel formula that consists of a mixture of Dirac masses, this model is a \emph{smooth} approximation which allows it to potentially generalize better. Moreover, as the HJB equation \eqref{eq:hjb} is the fundamental characterization of SGMs, we show in Proposition \ref{prop:hjbsolution} that the smooth WPO-informed kernel formula \eqref{eq:kernelcovdensity} is closed under the evolution of the HJB and, therefore, admits an analytical solution.

\begin{proposition}\label{prop:hjbsolution}
    The HJB equation \eqref{eq:hjb} with terminal condition $\hat{U}(x,T) = -\beta^2\log \hat{\pi}_\theta(x)$ with $\hat{\pi}_\theta(x)$ given by \eqref{eq:kernelcovdensity} is solved by
    \begin{align}\label{eq:hjbsolution}
        &\hat{U}(x,t) = -\beta^2\log\hat{\eta}_\theta(x,T-t) = - \beta^2 \log\left( \frac{1}{N} \sum_{i = 1}^N \frac{\det \bGamma_{T-t,\theta}(Z_i)}{(2\pi)^{d/2}}\exp\left( -\frac{(x-Z_i)^\top\bGamma_{T-t,\theta}(Z_i) (x-Z_i)}{2} \right) \right),\\
         & \bGamma_{t,\theta}(z) = (\bGamma_{\theta}(z) ^{-1} + \beta^2t\mathbf{I})^{-1}   . \nonumber 
    \end{align}
\end{proposition}
\begin{proof}
    It is easy to confirm that the terminal condition is satisfied exactly. Next, observe that
    \begin{align*}
        \partial_t& \hat{U}(x,t) = -\frac{\beta^2}{\hat{\eta}_\theta(x,T-t)} \frac{\partial}{\partial t} \hat{\eta}_\theta(x,T-t) = \frac{\beta^4}{\hat{\eta}_\theta(x,T-t)} \\ &  \times \sum_{i = 1}^N \frac{\det \bGamma_{T-t,\theta}(Z_i)}{N(2\pi)^{d/2}}\exp\left( - \frac{(x-Z_i)^\top\bGamma_{T-t,\theta}(Z_i) (x-Z_i)}{2} \right) \left( \frac{\|\bGamma_{T-t,\theta}(Z_i)(x-Z_i) \|^2 - \beta^2 \text{Tr}\bGamma_{\theta,T-t}}{2} \right),
    \end{align*}

    \begin{align*}
        \nabla \hat{U}(x,t) &= - \beta^2 \frac{\nabla {\hat{\eta}_\theta(x,T-t)}}{{\hat{\eta}_\theta(x,T-t)}}\\& = \frac{\beta^2}{\hat{\eta}_\theta(x,T-t)}\sum_{i = 1}^N \frac{\det \bGamma_{T-t,\theta}(Z_i)}{N(2\pi)^{d/2}}\exp\left( - \frac{(x-Z_i)^\top\bGamma_{T-t,\theta}(Z_i) (x-Z_i)}{2} \right)\left( \bGamma_{T-t,\theta}(Z_i)(x-Z_i)\right),
    \end{align*}

    \begin{align*}
        \frac{\beta^2}{2} \Delta  \hat{U}(x,t) = -\frac{\beta^4}{2}\left(\frac{\Delta \hat{\eta}_\theta(x,T-t) }{\hat{\eta}_\theta(x,T-t)} - \left\|\frac{\nabla \hat{\eta}_\theta(x,T-t)}{\hat{\eta}_\theta(x,T-t)} \right\|^2\right),
    \end{align*}

    where 
    \begin{align*}
        \Delta &\hat{\eta}_\theta(x,T-t) \\ &= \sum_{i = 1}^N \frac{\det \bGamma_{T-t,\theta}(Z_i)}{N(2\pi)^{d/2}}\exp\left(  - \frac{(x-Z_i)^\top\bGamma_{T-t,\theta}(Z_i) (x-Z_i)}{2} \right) \left( \|\bGamma_{T-t,\theta}(Z_i)(x-Z_i) \|^2 - \beta^2 \text{Tr}\bGamma_{\theta,T-t} \right).
    \end{align*}

    To differentiate the determinant of a symmetric matrix-valued function with respect to a parameter $t$, we used the fact that $\textbf{A}(t)$ for $t\in \R$, 
    $\partial_t \log \det \mathbf{A}(t) = \text{Tr}\left(\mathbf{A}(t)^{-1} \partial_t \mathbf{A}(t) \right)$ \cite{petersen2008matrix}. Matching terms, we see that $        \frac{\beta^2}{2} \Delta \hat{U}(x,t) = -\partial_t \hat{U}(x,t) + \frac{1}{2} \|\nabla \hat{U}(x,t)\|^2$.
\end{proof}

The WPO-informed kernel model presented in Proposition \ref{prop:hjbsolution} inherits the inductive bias of the score function as described by the kernel representation formula in Proposition \ref{prop:repformula}. See Figure \ref{fig:wposgm} for a graphical explanation of the relationship between the model in \eqref{eq:hjbsolution} and \eqref{eq:repformula}. The form of the WPO-informed kernel model \eqref{eq:kernelcovdensity} and the result of Proposition \ref{prop:hjbsolution} further supports the point that the model provides an \emph{explainable} characterization of SGMs. As noted in \cite{zhang2023mean} and \cite{lai2023fp}, SGMs inherently solve the HJB equation \eqref{eq:hjb}. Proposition \ref{prop:hjbsolution} shows that \emph{given} the terminal condition is a kernel formula (Gaussian mixture model), the score function is uniquely determined by the gradient of the solution \eqref{eq:hjbsolution}. Therefore, there is no need to perform score-matching for all $t \in [0,T]$, as it is done in all the score-matching objectives \eqref{eq:esm},\eqref{eq:ism},\eqref{eq:dsm} --- enforcing the terminal condition is sufficient to determine the score function $\hat{s}(x,t)$ for all $t\in [0,T]$. In the next section, we show that performing score-matching only for the terminal time $t = T$ in \eqref{eq:kernelcovdensity} is sufficient. Moreover, the implicit score-matching objective at the terminal time is used to train the local precision matrix model. The form of the model explicitly encodes the manifold learning properties of score-based generative models. In particular, we highlight Remark \ref{rmk:manifoldlearning} which provides a connection between the precision matrix in the WPO-informed model \eqref{eq:kernelcovdensity} and Riemannian manifolds.

\begin{remark}\textbf{(Learning the local precision matrices is manifold learning)} \label{rmk:manifoldlearning}
Incorporating a pre-processing step that learns the underlying data manifold, i.e. the latent space, has been empirically found to improve the generative qualities of SGMs \cite{vahdat2021latentSGM,rombach2021latentdiffusion}. Moreover \cite{pidstrigach2022score} observed and proved that SGMs inherently learn data manifolds. We explicitly exploit the results of \cite{pidstrigach2022score} by building a new kernel-based model \eqref{eq:scoremodel}  that learns information about the data manifold through local precision matrices. Manifolds embedded in Euclidean space are described by Riemannian metric tensors, which are a family of positive semidefinite symmetric matrices \cite{roweis2000nonlinear}. These metrics correspond to the learned precision matrices. The resulting density function is defined with respect to the Lebesgue measure of the data manifold.
\end{remark}

\paragraph{Enforcing the terminal condition with implicit score-matching efficiently learns the data manifold.}

Our kernel-based model is constructed to explicitly solve the HJB for $t\in[0,T)$; the remaining task is to learn the local precision matrices around the kernel centers so that the terminal condition is satisfied. To this end, we use the terminal condition to construct a loss function. Imposing some $L^2$ loss between the approximate $\hat{U}_\theta(x,T)$ and true $U(x,T)$ requires knowledge of the density and normalizing constant of $\pi$, which we do not have. Remarkably, however, the simplest way to enforce the terminal condition is to match the \emph{gradient} of the terminal condition, which is equivalent to implicit score-matching \eqref{eq:ism} only at the terminal time $t = T$. Score-matching is not needed for $t<T$ since the kernel solution to the HJB equation will optimize the ISM objective function exactly because of the mean-field game formulation. From Theorem 10 in \cite{zhang2023mean}, it is shown that ISM is equivalent to the SGM MFG optimality condition \eqref{eq:score_opt}. Moreover, Proposition \ref{prop:hjbsolution} shows that the kernel-based model solves the HJB exactly. 

Consider the optimization problem 
\begin{align}\label{eq:terminalcondition}
    \min_{\theta} \int_{\R^d} |\nabla U(x,T) - \nabla{\hat{U}}_\theta(x,T)|^2 \pi(x) \de x \implies \min_\theta  \int_{\R^d} \left(|\nabla \log \pi_\theta(x)|^2 + 2\Delta \log \pi_\theta(x) \right) \pi(x) \de x.
\end{align}
Solving this optimization problem has the effect of learning the local covariance matrix, which has the same effect as learning the manifold on which the data distribution is supported. The empirical kernel formula \eqref{eq:empiricaldistribution} is not well-defined as a function at the terminal time. Moreover, the derivative of the empirical kernel formula for $t = T$ is not defined, which implies that derivative cannot be taken and the integration by parts formula does not apply. There is, however, an implicit regularity assumption imposed on the target distribution when constructing SGMs since the integration by parts formula is used to derive the ISM \eqref{eq:ism}. The implicit regularity assumption is critical for the resulting generative model to generalize.

Let $y(\,\cdot\,;\theta):\R^d\to \R^N$ be a vector-valued function where the $i$-th component is 
\begin{align}
    {y}_i(x;\theta) = -(x-Z_i)^\top\bGamma_{\theta}(Z_i) (x-Z_i) +\log \det \bGamma_\theta(Z_i)- \frac{d}{2}\log 2\pi- \log N,
\end{align}
so that the density is $    \hat{\pi}_\theta(x) = \sum_{i = 1}^N \exp({y}_i(x;\theta))$, and the score function is
\begin{align}\label{eq:scoremodel}
    \hat{\mathsf{s}}_\theta(x)=\nabla \log \hat{\pi}_\theta(x) = \frac{\sum_{j = 1}^N \nabla {y}_j(x;\theta) \exp({y}_j(x;\theta))}{\sum_{j = 1}^N \exp({y}_j(x;\theta)) } = -\sum_{j = 1}^N \bGamma_\theta(Z_j)\left(x-Z_j \right) \sigma({y}(x;\theta))_j,
\end{align}
where $\sigma:\R^N \to \R^N$ is the softmax function, 
$    \sigma(y)_i = \frac{\exp(y_i)}{\sum_{j = 1}^N \exp(y_j)}.$ Matching the gradient of the terminal condition \eqref{eq:terminalcondition} and using the fact that $\Delta\log \pi_\theta= \pi_\theta^{-1}\Delta \pi_\theta-|\nabla\log\pi_\theta|^2$ yields the optimization problem
\begin{align}\label{eq:optimizationproblem}
&\min_\theta \int_{\R^d} \left( 2 \pi_\theta^{-1}\Delta \pi_\theta- |\nabla \log\pi_\theta(x)|^2 \right) \pi(x) \de x,
\end{align}
where
\begin{align}\label{eq:laplaciandensity}
    \pi_\theta(x)^{-1}\Delta \pi_\theta(x) = \sum_{i = 1}^N \left( |\nabla y_i(x;\theta)|^2 + \Delta y_i(x;\theta) \right) \sigma(y(x;\theta))_i, \\ \nabla y_i(x;\theta) = -\bGamma_\theta(Z_j)(x-Z_j),\,\,\,\, \Delta y_i(x;\theta) = -\text{Tr}(\bGamma_\theta(Z_j)).\nonumber
\end{align}

{A neural net is used to model $\bGamma_\theta(z)$ via the implicit score matching procedure at the terminal condition \eqref{eq:optimizationproblem}. This produces the WPO-informed kernel model for the score function \eqref{eq:kernelcovdensity} that, if used in conjunction with Proposition \ref{prop:hjbsolution} and the denoising SDE \eqref{eq:denoising}, will exactly produce samples from the kernel formula \eqref{eq:kernelcovdensity}. We note, however, in Remark \ref{remark:nosdes} that since \eqref{eq:kernelcovdensity} is a Gaussian mixture model, no simulation of SDEs is needed and samples can be produced directly from the WPO-informed kernel model. 
What allows the WPO-informed kernel formula to generalize and avoid memorization is its smoothness and regularity at the terminal condition. Moreover, we show that judicious use of a neural net can improve training --- rather than replacing the entire score function with a neural net, we preserve mathematical structure informed by the MFG PDEs associated with SGMs by embracing the kernel formula, and introduce a neural net to model the local precision matrices only. We demonstrate these claims via numerical examples in Section \ref{sec:numerical}. Moreover, in Section \ref{sec:bespoke} we show that PDE and kernel structure for the score function may yield informed neural net architecture for scalable implementations of the WPO-informed kernel model. }

\begin{remark}\textbf{(Implicit score-matching without autodifferentiation)}\label{remark:noautodiff}
In contrast to previous deep learning implementations of implicit score-matching \cite{song2020sliced}, which requires autodifferentiation packages and the Hutchinson estimator to compute the divergence of the score function, the WPO-kernel formula admits exact formulas for gradient, divergences, and Laplacians of the kernel formula and its Cole-Hopf transform. 

\begin{remark}\textbf{(No simulation of SDEs required)}\label{remark:nosdes}
The kernel-based model \eqref{eq:kernelcovdensity} obviates the need for simulating SDEs since the mixture model can be sampled directly. One, however, can still derive a score model for the denoising SDE with standard choices of the noising process (similar to Proposition \ref{prop:hjbsolution}). The noising process is typically chosen to be a linear SDE, so the transition kernel is Gaussian, and the score function can be found by convolving the transition kernel with $\hat{\pi}_\theta(x)$. 
\end{remark}

\end{remark}

\paragraph{Parametrizing the precision matrix.}
We outline one (but not the only) approach for parameterizing the precision matrix \cite{pinheiro1996unconstrained} which we apply in our numerical examples in Section \ref{sec:numerical}. We parametrize the precision matrix $\mathbf{\Gamma}_\theta \in \R^{d\times d}$ in terms of the Cholesky factors: $\mathbf{\Gamma}_\theta(x) = \mathbf{L}_\theta(x) \mathbf{L}_\theta(x)^\top$, where $\mathbf{L}_\theta$ is a lower triangular matrix. The entries of the Cholesky factor is populated by the outputs of a feedforward neural network $\psi_\theta: \R^d \to \R^{d(d+1)/2}$.

We summarize the algorithm in Algorithm \ref{alg:wpokernelscore}. The kernel model can be directly sampled just as a Gaussian mixture model would. This algorithm is an explainable reformulation of score-based generative modeling, derived by explicitly incorporating the inherent kernel structure. The model clarifies the manifold learning properties of SGM, the use of early stopping, and obviates the need for simulating the denoising SDE.

\begin{algorithm}
  \caption{Learning WPO-informed kernel models}
  \begin{algorithmic}
      \STATE \textbf{Input}: Samples $\{Z_i\}\sim \pi$, Cholesky factor entries neural network $\psi_\theta: \R^d \to \R^{d(d+1)/2}$
      \STATE Define precision matrix $\bGamma_\theta(x) = \mathbf{L}_\theta(x) \mathbf{L}_\theta(x)^\top$
        \STATE Set $\hat{\pi}_\theta \gets \eqref{eq:kernelcovdensity}$,$\nabla \log \hat{\pi}_\theta \gets \eqref{eq:scoremodel}$,  $ \hat{\pi}_\theta^{-1}\Delta \hat{\pi}_\theta \gets \eqref{eq:laplaciandensity}$ (\text{Requires no autodifferentiation})
      \STATE Find the optimal precision matrix
       $  \theta^* \gets \arg\min_\theta \frac{1}{N} \sum_{i = 1}^N 2\hat{\pi}_\theta(Z_i)^{-1}\Delta \hat{\pi}_\theta(Z_i)- |\nabla \log \hat{\pi}_\theta(Z_i)|^2 $\\
      \RETURN $\hat{\pi}_{\theta^*}(x)$
  \end{algorithmic}
  \label{alg:wpokernelscore}
\end{algorithm}

\begin{remark}{(Early stopping revisited)}
    Early stopping is a frequently used strategy to improve the sample quality of score-based generative models \cite{li2024good}. When simulating the denoising SDE, each trajectory is not simulated for the full time interval and is stopped early. This has the effect of smoothing the generative distribution and has been observed to aid in generalization. The precise amount of early stopping requires tuning. We may interpret our kernel-based formula as a generalization of early stopping, in which the local precision matrices are learned from data. It can also be interpreted as a way to choose the optimal early stopping parameter. Consider a simplified model where the local covariance matrix is identical over all space and is simply $h\mathbf{I}$. Then the optimal covariance solves \eqref{eq:optimizationproblem} where $\theta^* = h^*$ with the score model \eqref{eq:scoremodel} where $\bGamma = {h^*}^{-1}\mathbf{I}$. Sampling from this tuned distribution is the same as early stopping at $t = h$. 
\end{remark}

\section{Numerical examples}
\label{sec:numerical}

We conduct illustrative numerical experiments on synthetic datasets to demonstrate the effectiveness of our WPO-informed kernel model. We emphasize that the method here is not optimized to be scalable--- rather we implement the kernel model to demonstrate its manifold learning and generalization properties. In particular, we show that the WPO kernel model trains faster and intrinsically provides a density estimate. We also illustrate the explicit manifold learning properties of the kernel formula and demonstrate its superiority to early stopping. 

\subsection{WPO-informed kernel model trains faster and provides density estimation}
We implement the WPO-informed kernel model according to Algorithm \ref{alg:wpokernelscore}, where the function $\psi_\theta$ that models the entries of the Cholesky factor comprise is a feedforward neural network with five hidden layers of 64 nodes and a GeLU activation function. The terminal score-matching optimization problem is solved via stochastic gradient descent with batch size 64 and training dataset of size $N_\text{train} = 5\times 10^4$. The kernel centers $\{Z_i\}_{i = 1}^N$ where $N = 5000$ are chosen randomly from the training data. For comparison, we also train a standard SGM \eqref{eq:denoising} with the denoising score-matching objective \eqref{eq:dsm} to train the score neural network. The neural network for the score function $\mathsf{s}_\theta(x,t)$ has the same hidden architecture as $\psi_\theta$ and is trained with the same $5\times10^4$ training samples. 

In Figure \ref{fig:comparison_toydata}, we show the generated samples from the two models and compare it with samples from the true data distribution. Observe in Figure \ref{fig:wpo2dsamples} that the WPO-informed kernel model is able to nearly exactly reproduce the true distribution with just $5\times10^4$ steps of SGD. In contrast, in the Figure \ref{fig:dsm50000}, we show that the SGM learned via denoising score matching produces poor quality samples in $5\times 10^4$ steps, while Figure \ref{fig:dsm1000000} shows that it can reproduce the samples after $10^6$ training steps. Notice, however, that the two thin moon dataset is poorly approximated by denoising score matching SGM even after $10^6$ training steps. This is likely due to the fact that the neural net used to approximate the score function is unable to capture the irregular score function that arises from the thin moons. In contrast, the WPO-informed kernel model easily models the low-dimensional nature of the thin moons.

\begin{figure}

\begin{subfigure}{\textwidth} 
    \includegraphics[width=0.120\linewidth,trim=58 38 40 40, clip]{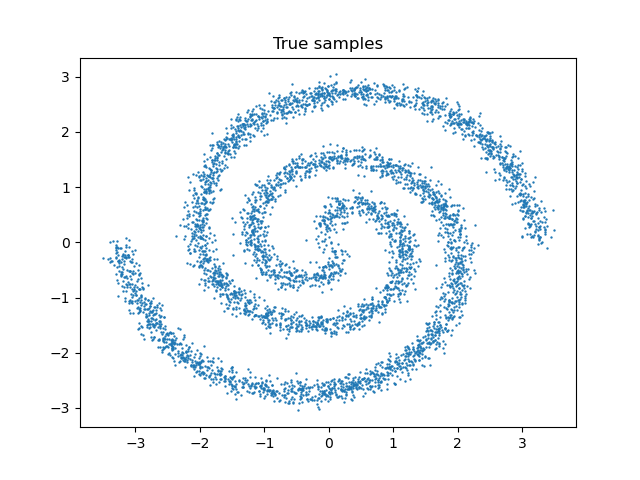}
    \includegraphics[width=0.120\linewidth,trim=58 38 40 40, clip]{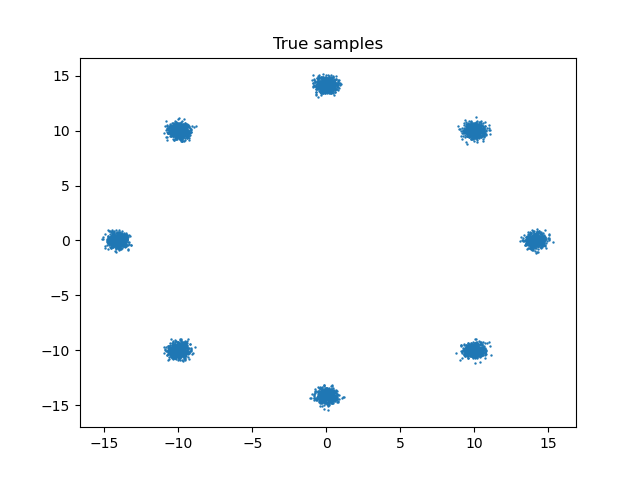}
    \includegraphics[width=0.120\linewidth,trim=58 38 40 40, clip]{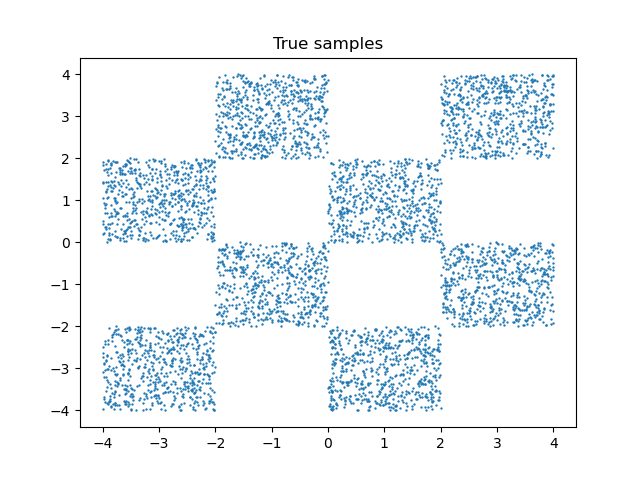}
    \includegraphics[width=0.120\linewidth,trim=58 38 40 40, clip]{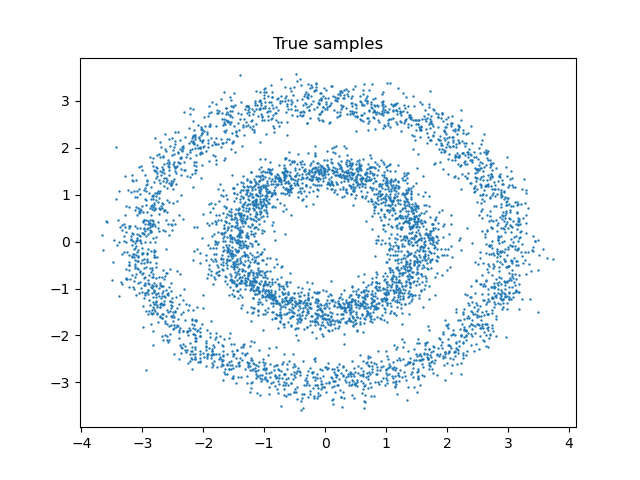}
    \includegraphics[width=0.120\linewidth,trim=58 38 40 40, clip]{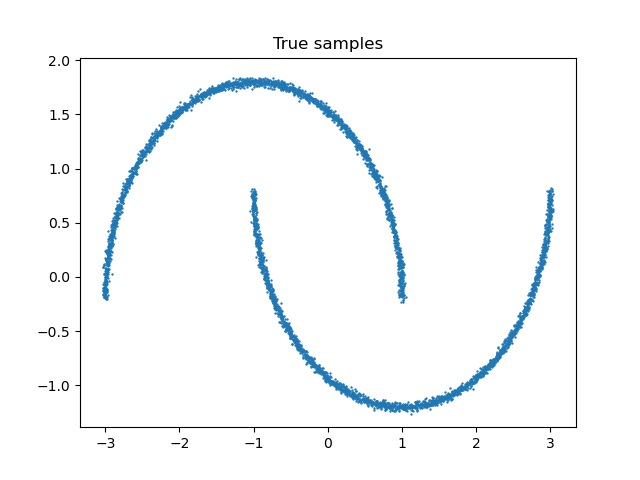}
    \includegraphics[width=0.120\linewidth,trim=58 38 40 40, clip]{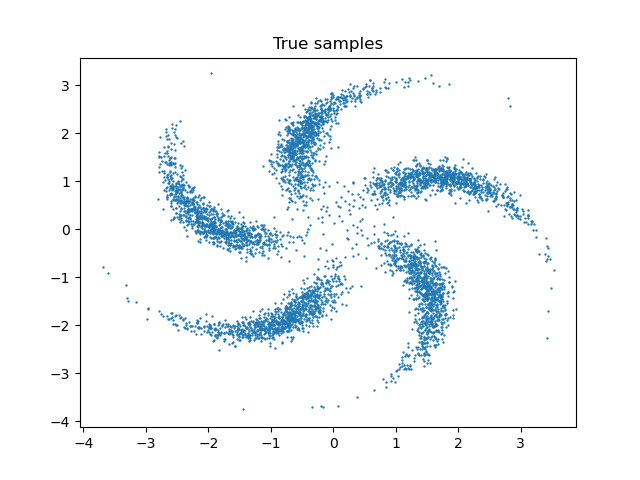}
    \includegraphics[width=0.120\linewidth,trim=58 38 40 40, clip]{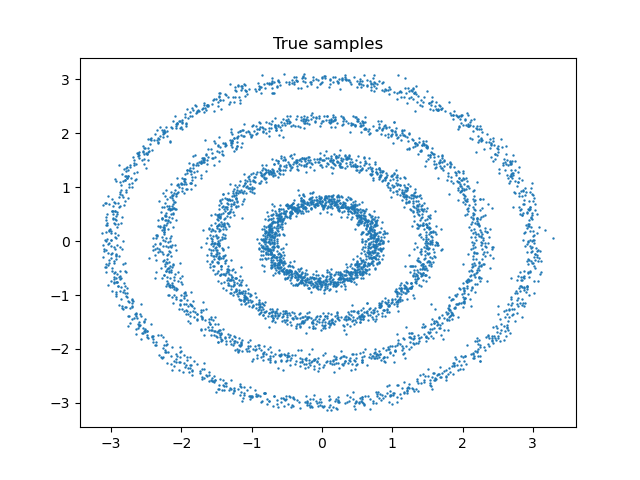}
    \includegraphics[width=0.120\linewidth,trim=58 38 40 40, clip]{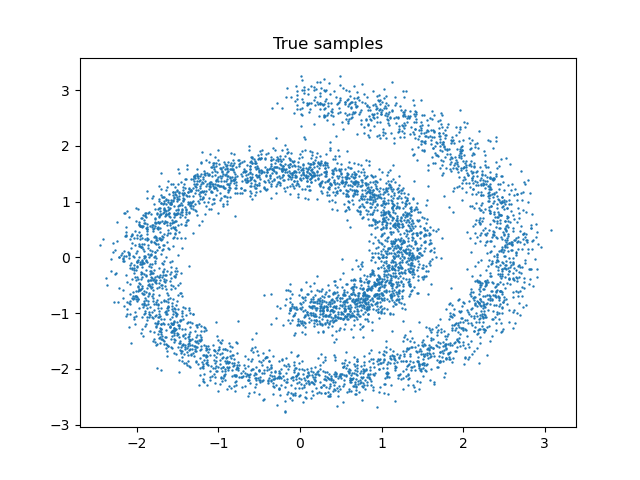}
\caption{True samples}\label{fig:truesamples}
  \end{subfigure}
  
\begin{subfigure}{\textwidth}
    \includegraphics[width=0.120\linewidth,trim=54 39 43 41, clip]{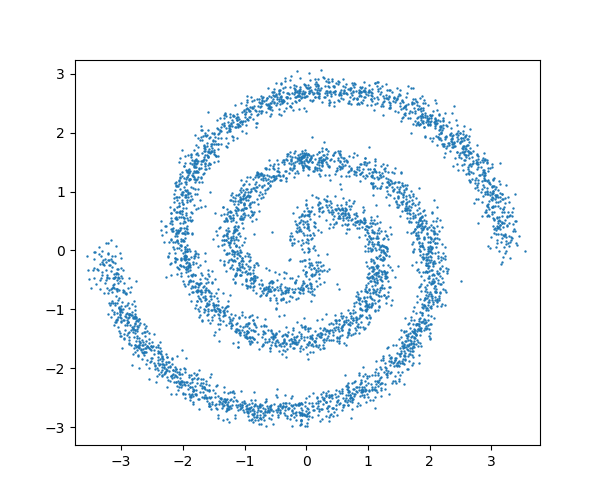}
    \includegraphics[width=0.120\linewidth,trim=54 39 43 41, clip]{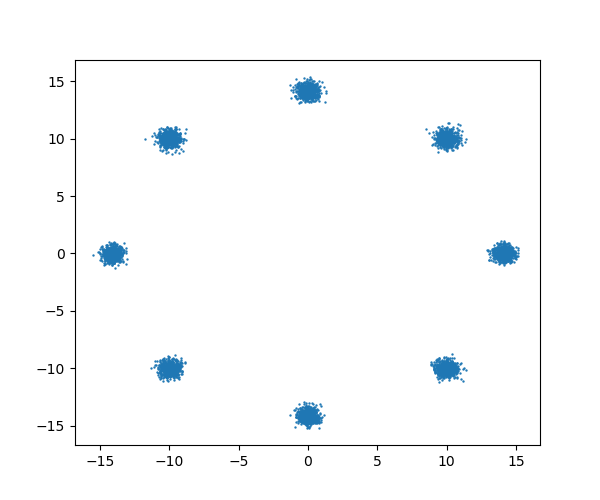}
    \includegraphics[width=0.120\linewidth,trim=54 39 43 41, clip]{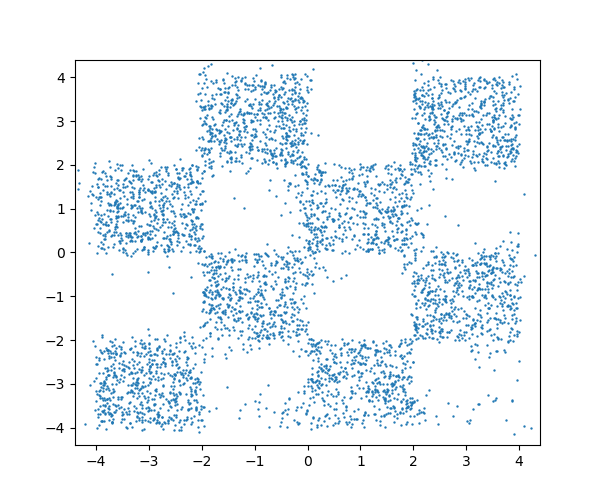}
    \includegraphics[width=0.120\linewidth,trim=54 39 43 41, clip]{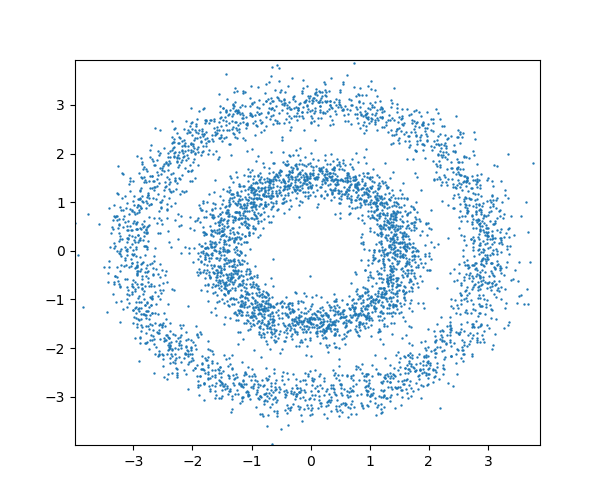}
    \includegraphics[width=0.120\linewidth,trim=54 39 43 41, clip]{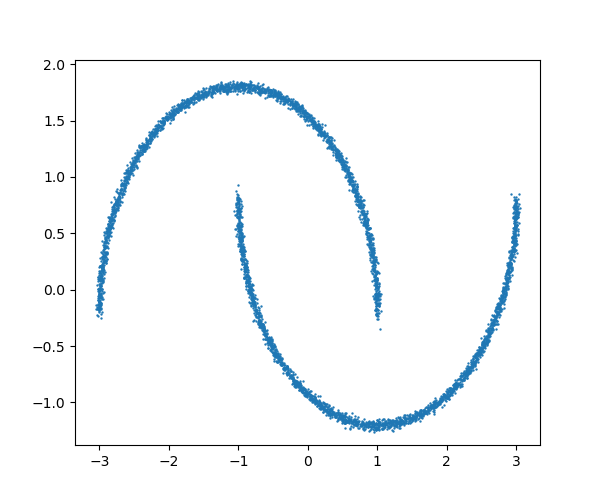}
    \includegraphics[width=0.120\linewidth,trim=54 39 43 41, clip]{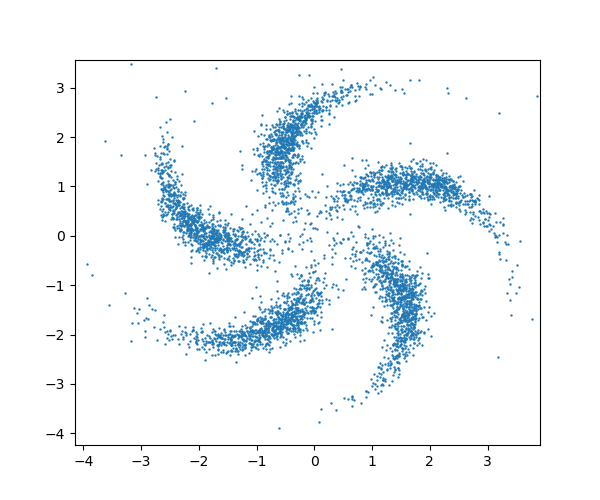}
    \includegraphics[width=0.120\linewidth,trim=54 39 43 41, clip]{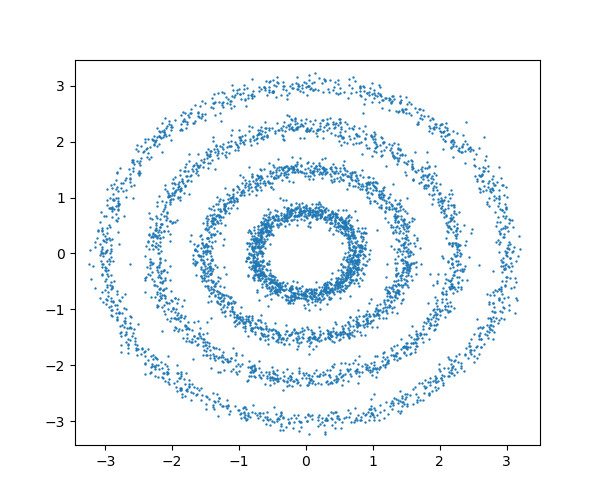}
    \includegraphics[width=0.120\linewidth,trim=54 39 43 41, clip]{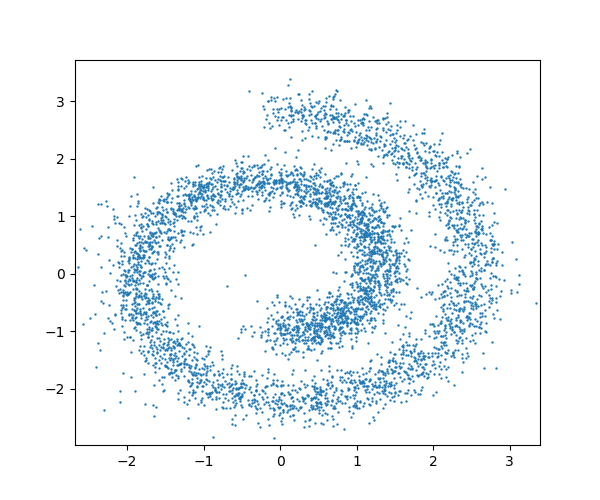}
    \caption{WPO-informed kernel model, trained with $5\times 10^4$ epochs}\label{fig:wpo2dsamples}
\end{subfigure}


\begin{subfigure}{\textwidth}
    \includegraphics[width=0.120\linewidth,trim=58 38 40 40, clip]{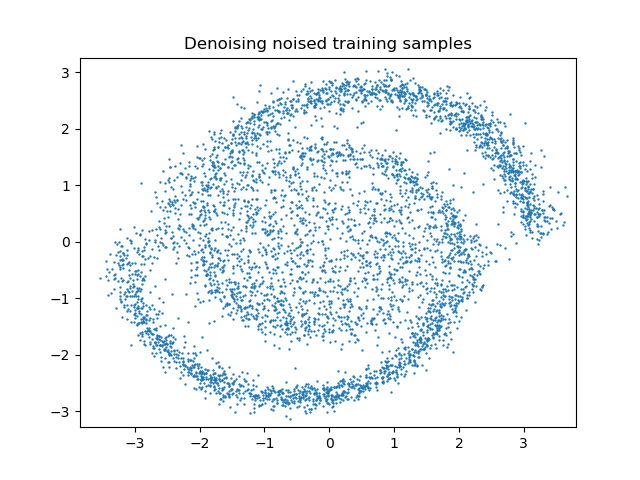}
    \includegraphics[width=0.120\linewidth,trim=58 38 40 40, clip]{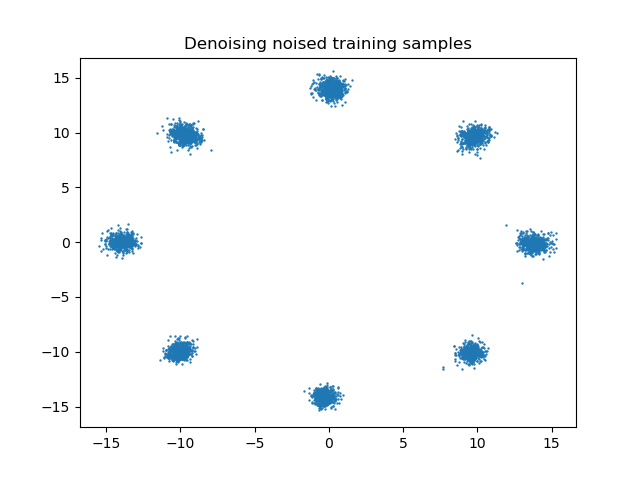}
    \includegraphics[width=0.120\linewidth,trim=58 38 40 40, clip]{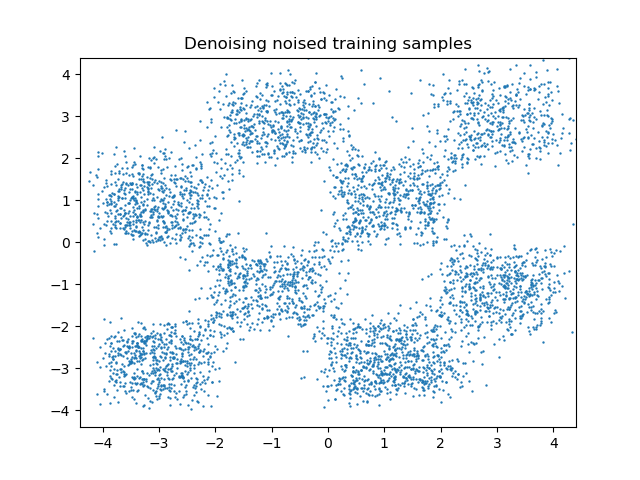}
    \includegraphics[width=0.120\linewidth,trim=58 38 40 40, clip]{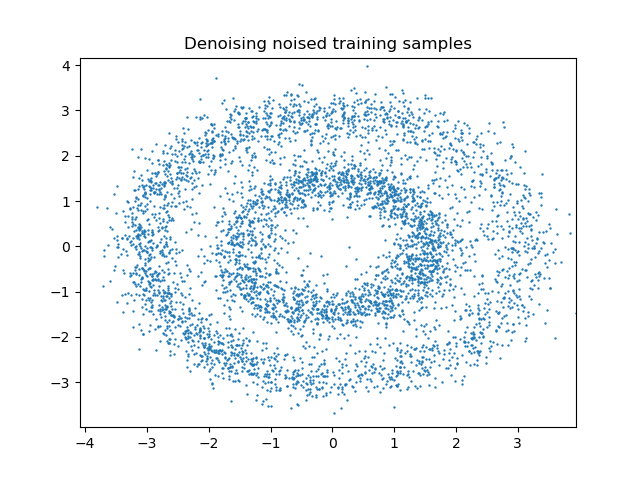}
    \includegraphics[width=0.120\linewidth,trim=58 38 40 40, clip]{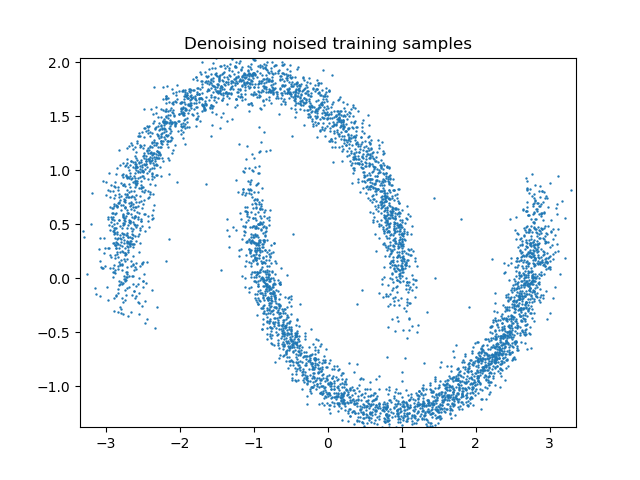}
    \includegraphics[width=0.120\linewidth,trim=58 38 40 40, clip]{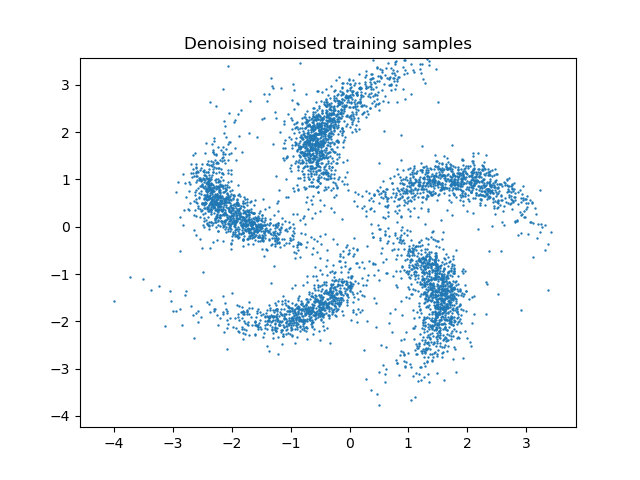}
    \includegraphics[width=0.120\linewidth,trim=58 38 40 40, clip]{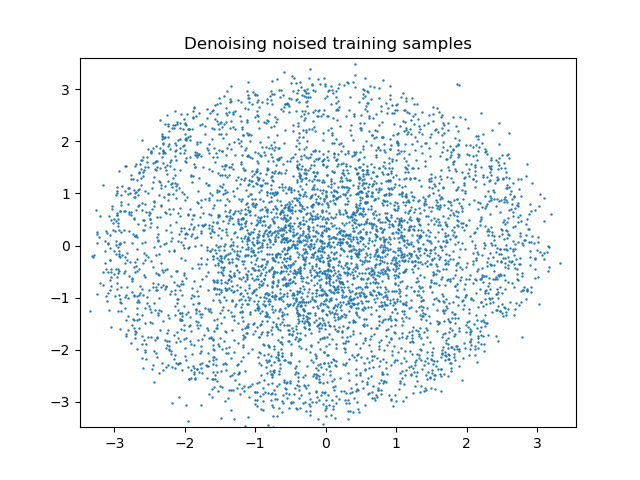}
    \includegraphics[width=0.120\linewidth,trim=58 38 40 40, clip]{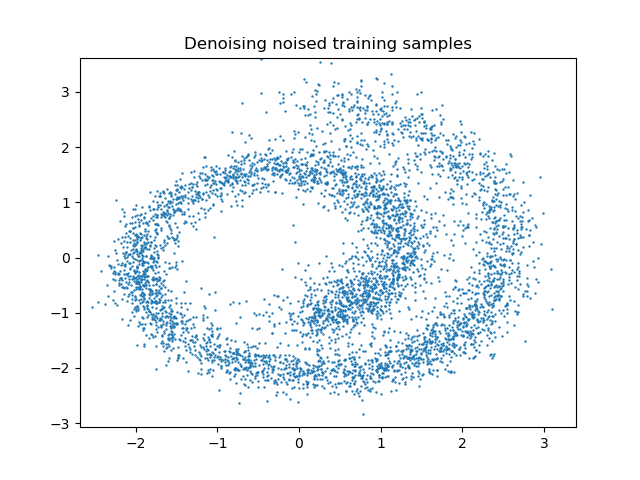}
    \caption{Denoising score matching, trained with $5\times 10^4$ epochs}\label{fig:dsm50000}
\end{subfigure}


\begin{subfigure}{\textwidth}
    \includegraphics[width=0.120\linewidth,trim=58 38 40 40, clip]{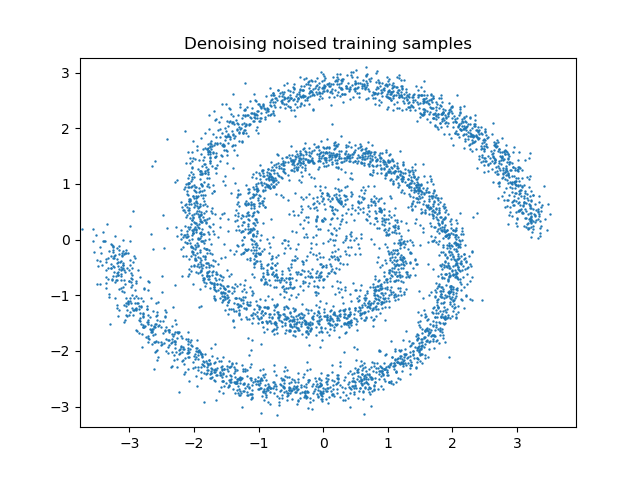}
    \includegraphics[width=0.120\linewidth,trim=58 38 40 40, clip]{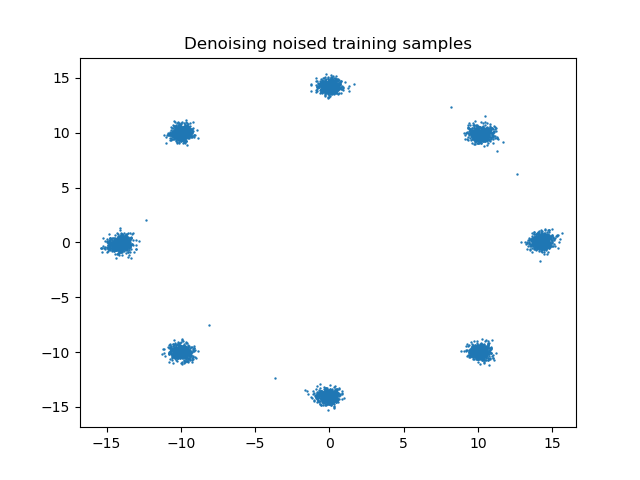}
    \includegraphics[width=0.120\linewidth,trim=58 38 40 40, clip]{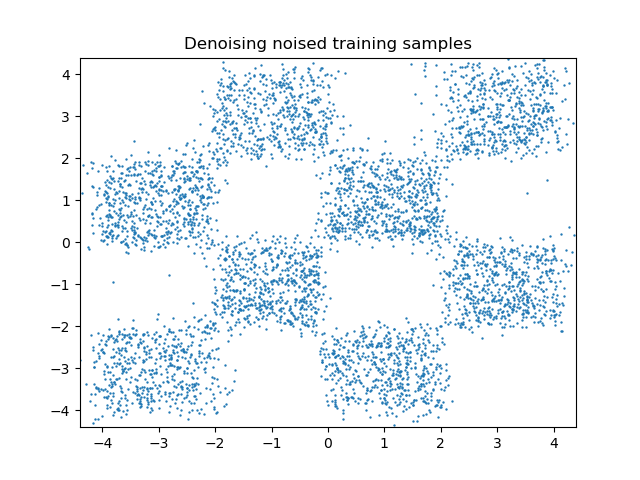}
    \includegraphics[width=0.120\linewidth,trim=58 38 40 40, clip]{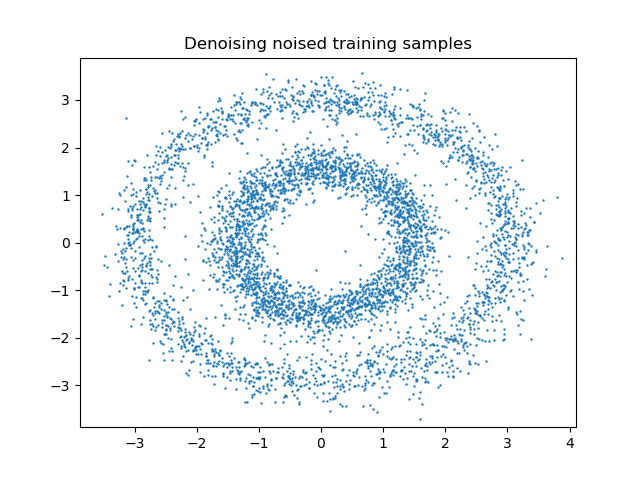}
    \includegraphics[width=0.120\linewidth,trim=58 38 40 40, clip]{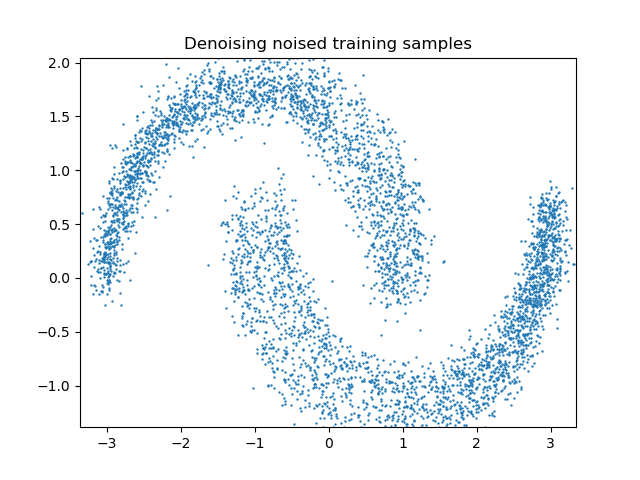}
    \includegraphics[width=0.120\linewidth,trim=58 38 40 40, clip]{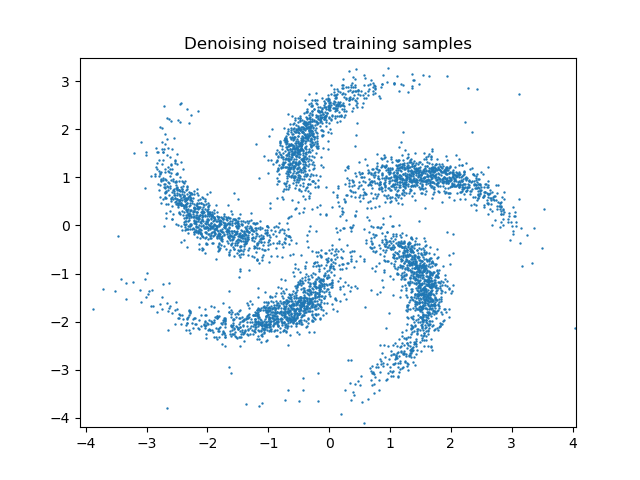}
    \includegraphics[width=0.120\linewidth,trim=58 38 40 40, clip]{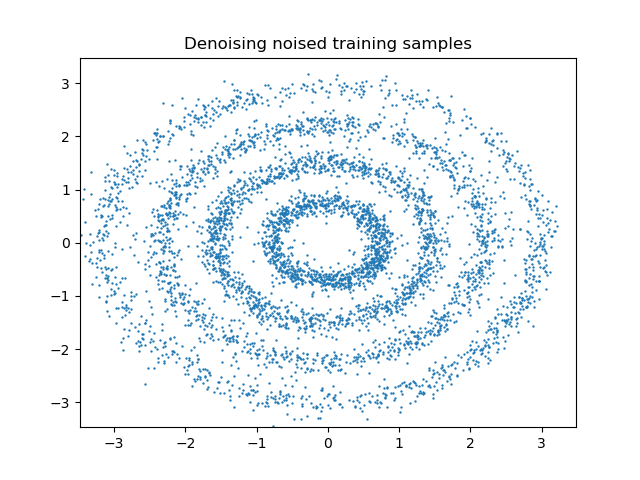}
    \includegraphics[width=0.120\linewidth,trim=58 38 40 40, clip]{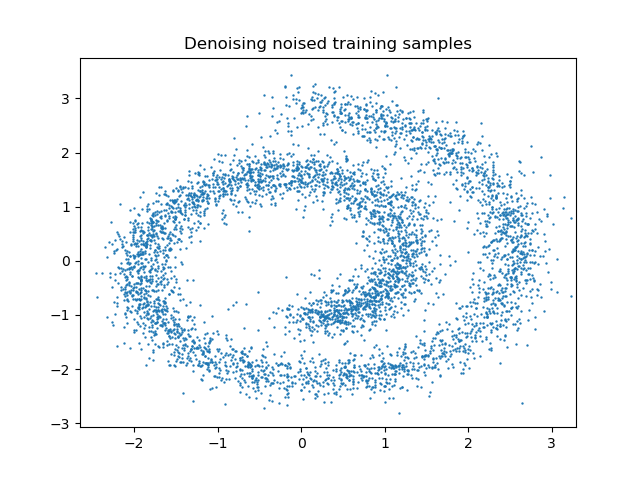}
    \caption{Denoising score matching, trained with $10^6$ epochs}\label{fig:dsm1000000}
\end{subfigure}

  \caption{2500 samples generated via different models when training dataset size is limited to $5 \times 10^4$. } 
  \label{fig:comparison_toydata} 
\end{figure}

Furthermore, in contrast to standard  score-based generative models, the WPO-informed kernel model inherently provides a density estimate. In Figure \ref{fig:wpo_densiies} we plot the kernel density approximation that corresponds with the generated samples Figure \ref{fig:6Dswissroll}. We emphasize that these density plots are \emph{not} reconstructed densities from the generated samples, rather they are directly provided by the WPO-informed kernel model in \eqref{eq:kernelcovdensity}.

\begin{figure}
  \begin{minipage}{0.25\textwidth}
    \centering
    \includegraphics[width=\linewidth,trim=0 0 0 0, clip]{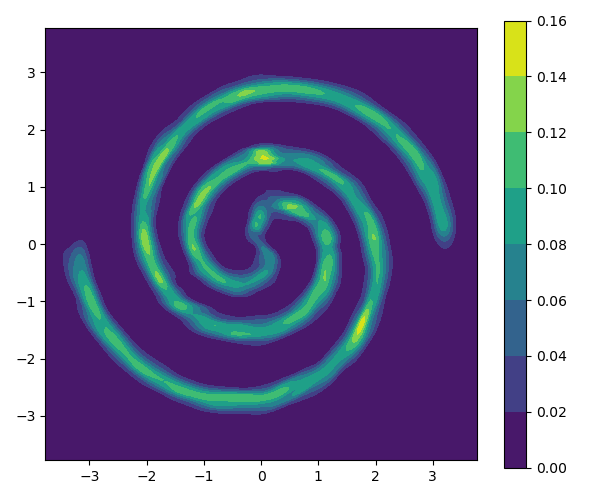}
  \end{minipage}%
  \begin{minipage}{0.25\textwidth}
    \centering
    \includegraphics[width=\linewidth,trim=0 0 0 0, clip]{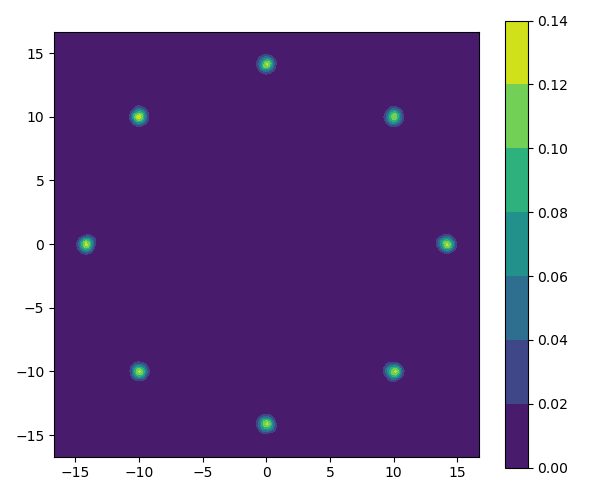}
  \end{minipage}%
  \begin{minipage}{0.25\textwidth}
    \centering
    \includegraphics[width=\linewidth,trim=0 0 0 0, clip]{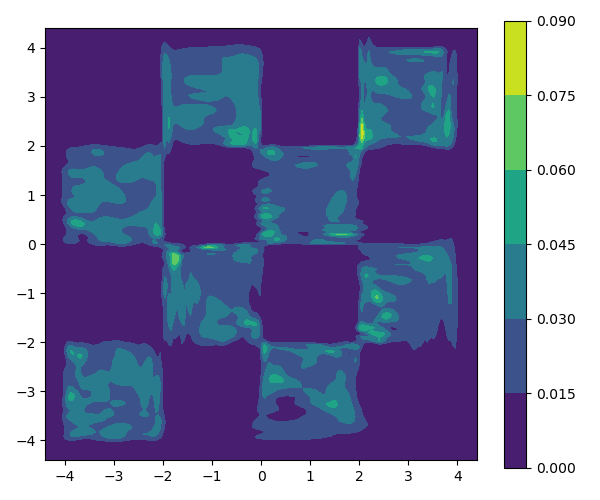}
  \end{minipage}%
  \begin{minipage}{0.25\textwidth}
    \centering
    \includegraphics[width=\linewidth,trim=0 0 0 0, clip]{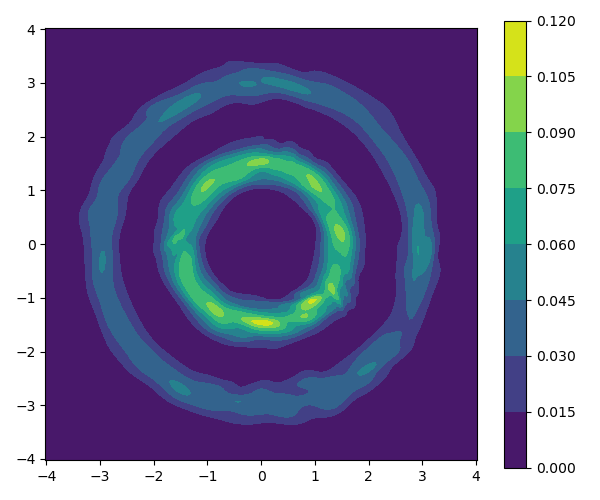}
  \end{minipage}
    \begin{minipage}{0.25\textwidth}
    \centering
    \includegraphics[width=\linewidth,trim=0 0 0 0, clip]{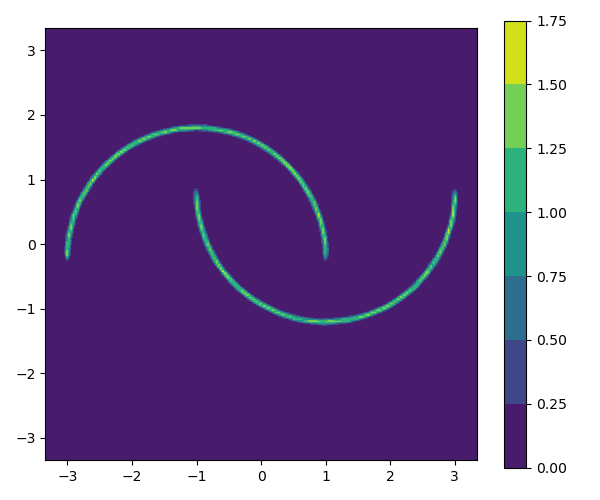}
  \end{minipage}%
  \begin{minipage}{0.25\textwidth}
    \centering
    \includegraphics[width=\linewidth,trim=0 0 0 0, clip]{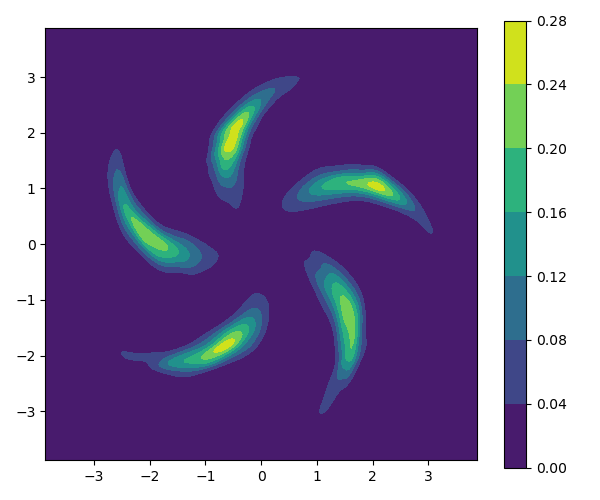}
  \end{minipage}%
  \begin{minipage}{0.25\textwidth}
    \centering
    \includegraphics[width=\linewidth,trim=0 0 0 0, clip]{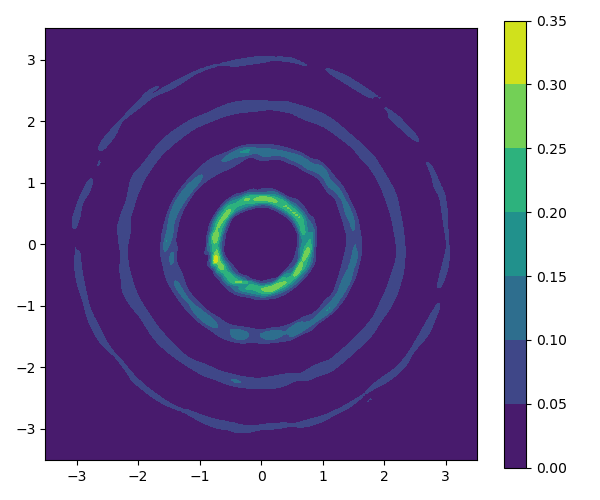}
  \end{minipage}%
  \begin{minipage}{0.25\textwidth}
    \centering
    \includegraphics[width=\linewidth,trim=0 0 0 0, clip]{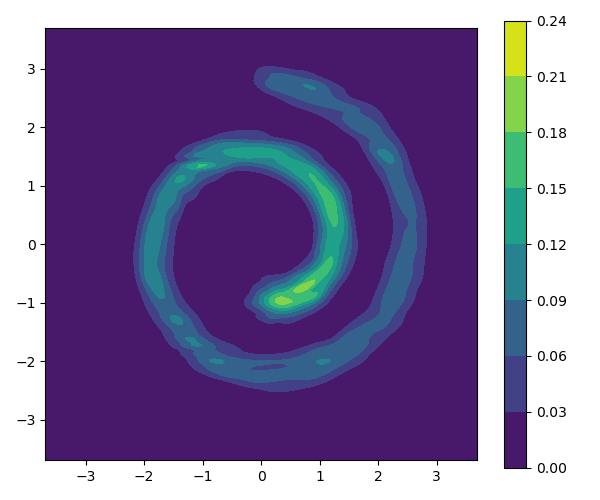}
  \end{minipage}%
  \caption{Density plots constructed by evaluating the kernel density estimated using WPO-informed kernel model as in experiments shows in Figure~\ref{fig:comparison_toydata}. These density plots are not reconstructed densities from samples. }
  \label{fig:wpo_densiies}
\end{figure}

To demonstrate that the WPO-informed kernel model is at least scalable to moderate dimensions even with the crude parametrization of the Cholesky factors of the precision matrices, we apply our method to the 3D swissroll dataset noisily embedded in a six-dimensional space. In Figure \ref{fig:6Dswissroll} we demonstrate the WPO-informed kernel model is able to produce a density estimate of the data distribution.

\begin{figure}

\begin{subfigure}{0.49\textwidth}
        \includegraphics[width = \textwidth]{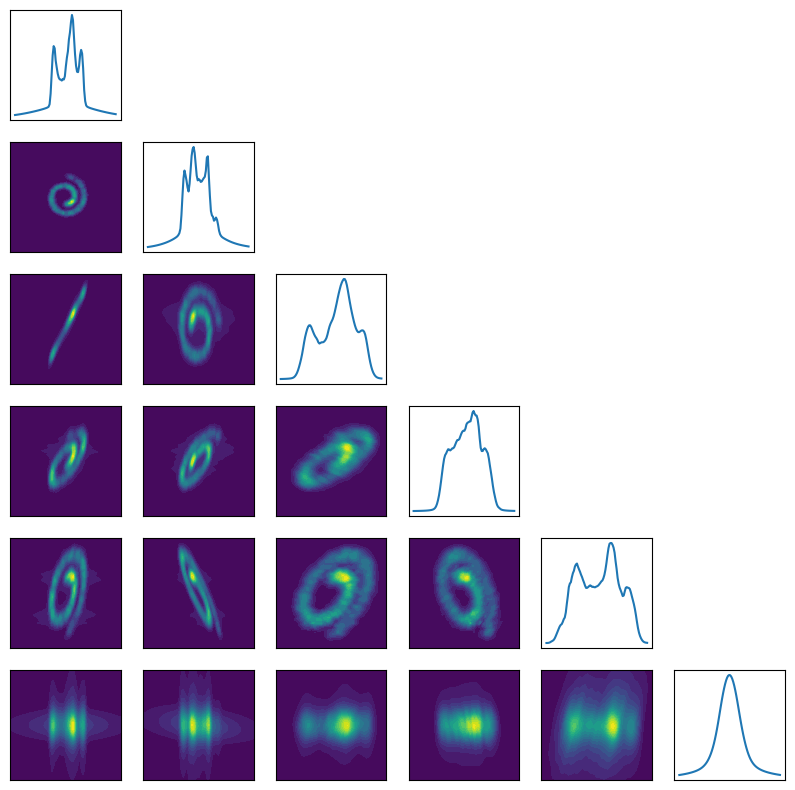}
        \subcaption{WPO-informed kernel model 2D marginal densities.  }
\end{subfigure}
\begin{subfigure}{0.49\textwidth}
        \includegraphics[width = \textwidth]{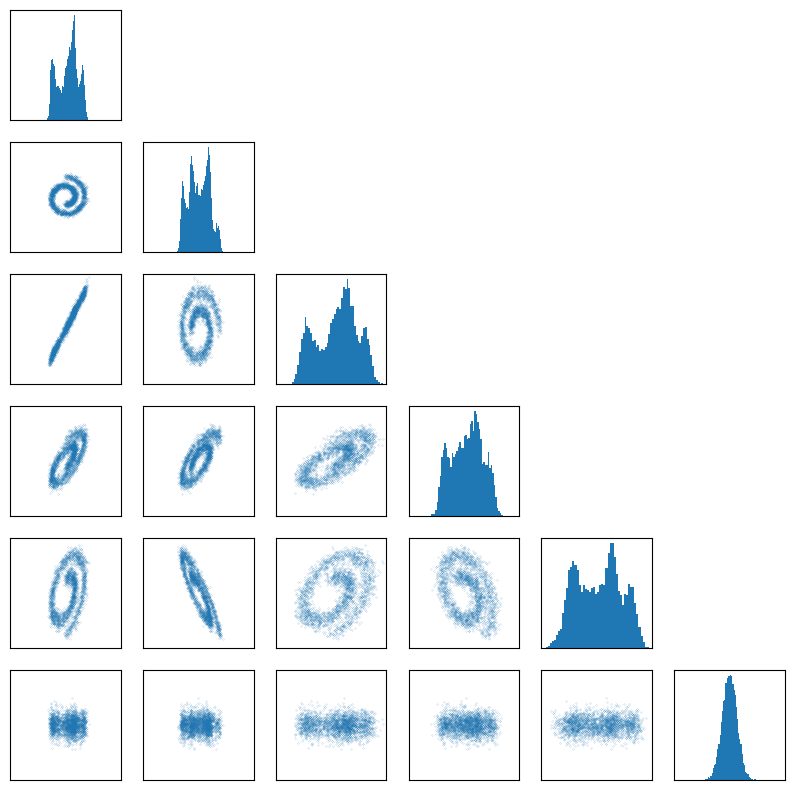}
        \subcaption{True samples.}
\end{subfigure}
        \caption{Six dimensional example: 3D swissroll noisily embedded in a 6D space. Proof of concept that the WPO-informed kernel model \eqref{eq:kernelcovdensity} is at least scalable to moderate dimensions. } We train the WPO-informed kernel model with 5000 kernel centers with $10^6$ training points. A Cholesky factor model is a feedforward neural net with 6 hidden layers and 64 nodes per layer. The model is trained via stochastic gradient descent with batch size 64 over $10^5$ iterations. The 2D marginal densities are not reconstructed from data. 
        \label{fig:6Dswissroll}
\end{figure}

\subsection{Learning the local covariance learns the data manifold}
To highlight the manifold learning properties of the WPO-informed kernel model, in Figure \ref{fig:local_cov_2moon} we plot the density estimate of the two thin moons example along with plots of evaluations of the learned local covariance (inverse precision) matrices. We plot the eigenvectors of each covariance matrix scaled by its corresponding eigenvalue. Notice that the orientation of the ellipses change with location and that the axes of the ellipse are not identical. The precision matrix that is learned is, in effect, learning the Riemannian metric of the data manifold \cite{roweis2000nonlinear}. Moreover, as the data distribution truly becomes lower dimensional, i.e., the two moons become simply two lines, then the length of the shorter axis of each ellipse will tend towards zero.

\begin{figure}
  \begin{minipage}{0.5\textwidth}
    \centering
    \includegraphics[width=\linewidth]{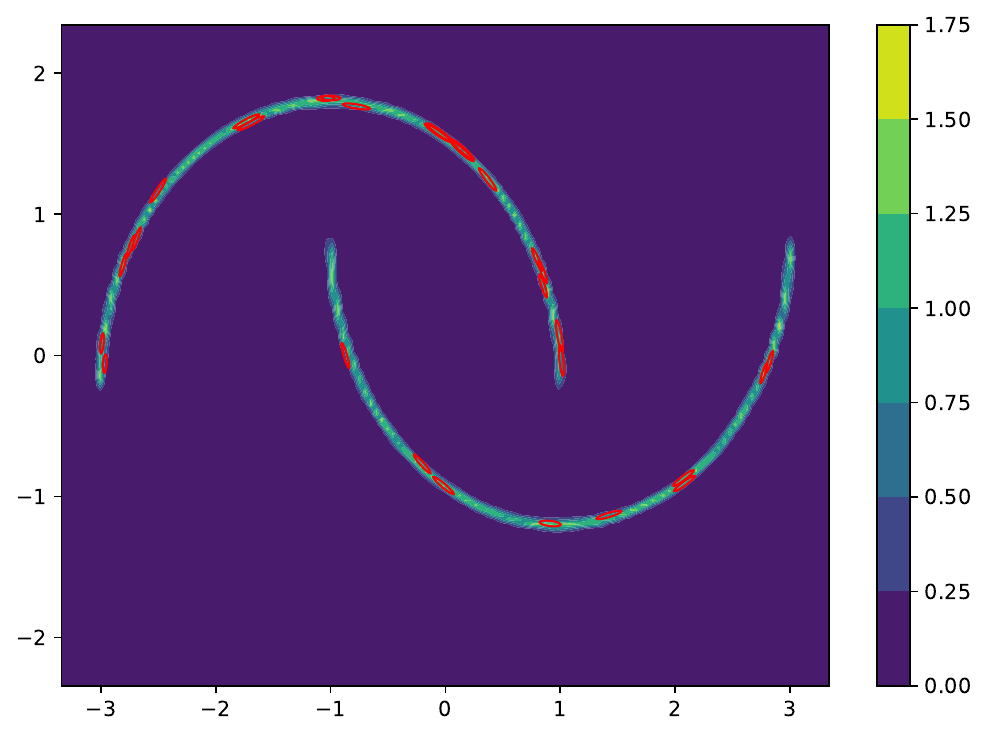}
  \end{minipage}%
    \begin{minipage}{0.5\textwidth}
    \centering
    \includegraphics[width=\linewidth]{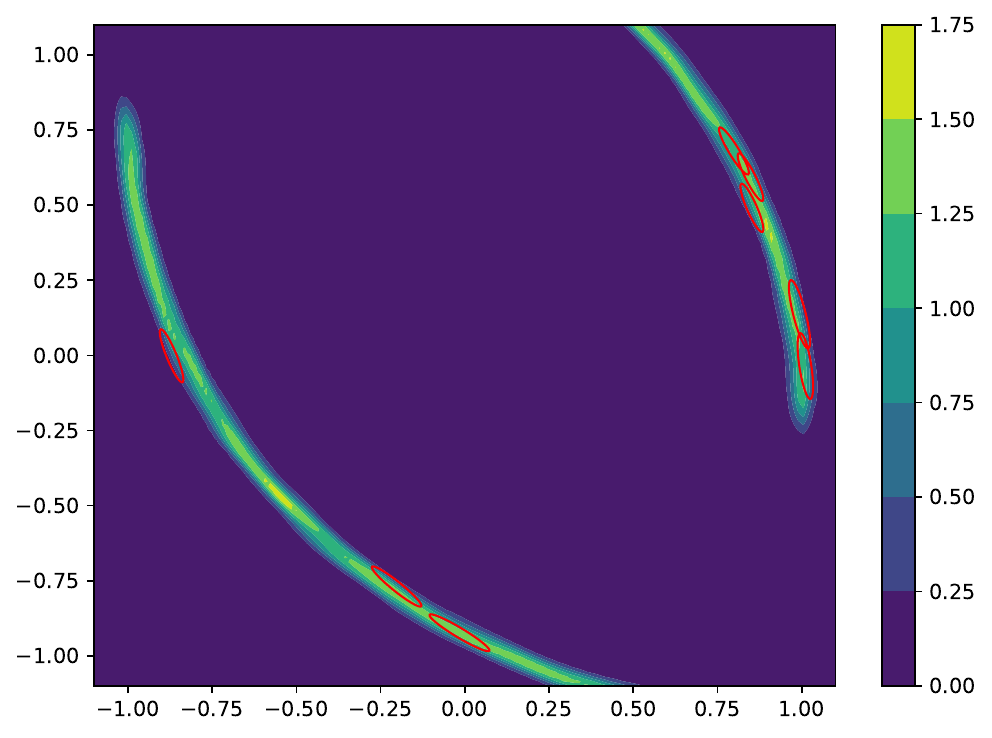}
  \end{minipage}%
    \caption{We use ellipses to represent the local covariance matrices obtained from the WPO-informed kernel model, which shows our trained model reveals the underlying data manifold. The model is trained with the two moons dataset, as the experiment depicted in Figure \ref{fig:wpo_densiies}. We draw $25$ samples $Z_i, i = 1, ..., 25$ from the trained kernel centers, accessing the learned covariance matrices by evaluating the trained Cholesky factorization $\bGamma_{\theta}(x)^{-1}$ at $x = Z_i$ for $i = 1, ..., 25$. Each ellipse is centered at $Z_i$, and its orientation and axes are determined by the eigenvectors and eigenvalues of the corresponding covariance matrix. On the right, we present a zoomed-in plot of the left figure.}
  \label{fig:local_cov_2moon}
\end{figure}

\subsection{Manifold learning generalizes better than early stopping {and prevents memorization}}

\label{sec:num:earlystop}

We further emphasize that the proposed WPO-informed kernel model generalizes better than the empirical kernel formula \eqref{eq:overfitsgm} due to the WPO-informed model's manifold learning capabilities. From \cite{li2024good,pidstrigach2022score}, it is known that a neural net trained with the denoising score matching objective \eqref{eq:dsm} will, as the model complexity and training dataset increases, converge to the kernel score formula in \eqref{eq:overfitsgm}, which is known to {produce a generative model that memorizes the training data and will not generalize.} However, in practical implementation, early stopping of the denoising diffusion process, i.e., simulating \eqref{eq:denoising} for $t\in [0,T-\epsilon]$ for some $0<\epsilon\ll 1$ is often performed so that the SGM generalizes. As discussed in Section \ref{sec:simplekernel} and in \cite{li2024good}, in the continuous time, early stopping has the effect of sampling from $(G_{\beta^2/2,\epsilon}* \hat{\pi})(x)$. 

In Figure \ref{fig:wpo_vs_const_kernel} we perform an illustrative experiment showing the superior generalization capabilities of the WPO-informed kernel model. Sampling from $(G_{\beta^2/2,\epsilon} * \hat{\pi})(x)$ is effectively sampling from the original training dataset added to an isotropic Gaussian noise. When $\epsilon$ is small, the generative distribution is quite rough and has patchy support. When $\epsilon$ is large, this has the effect of sampling from a very smoothed data distribution, which means the model lacks the expressiveness to approximate the true data distribution. In contrast, we plot the result of the WPO-informed kernel model and show that it is able to adaptively learn the local precision matrices and, therefore, better approximate the data distribution and generalize better.

\begin{figure}[h]

  \begin{minipage}{0.1\textwidth}
$\epsilon= 0.05$
  \end{minipage}%
  \begin{minipage}{0.22\textwidth}
    \centering
    \includegraphics[width=\linewidth,trim=0 0 0 0, clip]{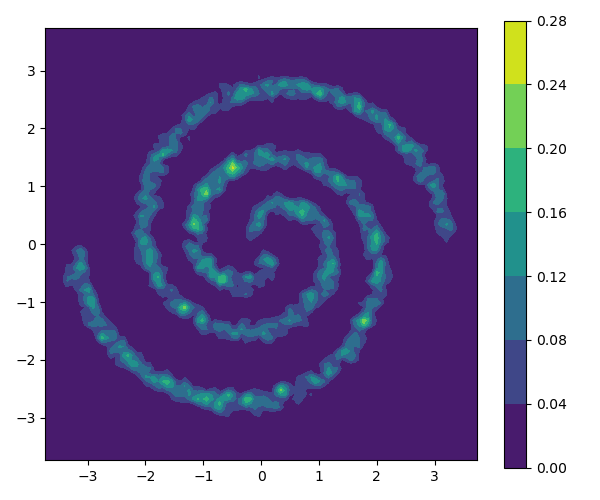}
  \end{minipage}%
  \begin{minipage}{0.22\textwidth}
    \centering
    \includegraphics[width=\linewidth,trim=0 0 0 0, clip]{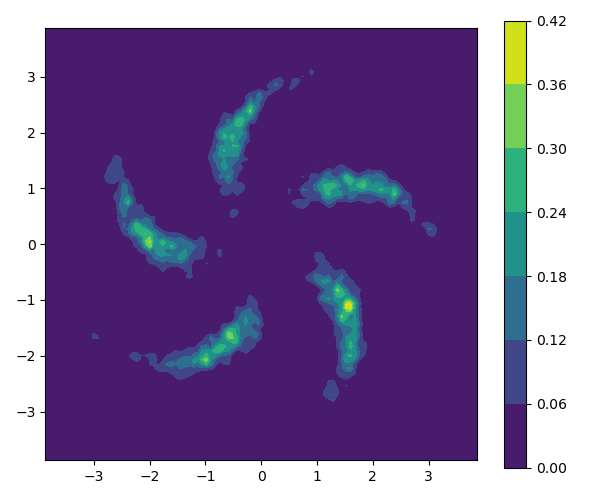}
  \end{minipage}%
  \begin{minipage}{0.22\textwidth}
    \centering
    \includegraphics[width=\linewidth,trim=0 0 0 0, clip]{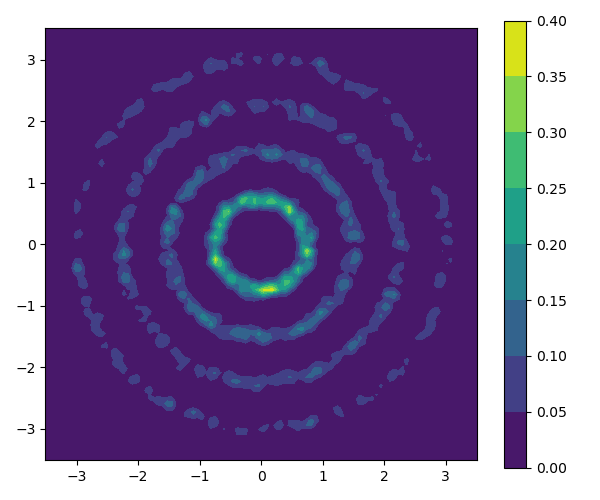}
  \end{minipage}%
  \begin{minipage}{0.22\textwidth}
    \centering
    \includegraphics[width=\linewidth,trim=0 0 0 0, clip]{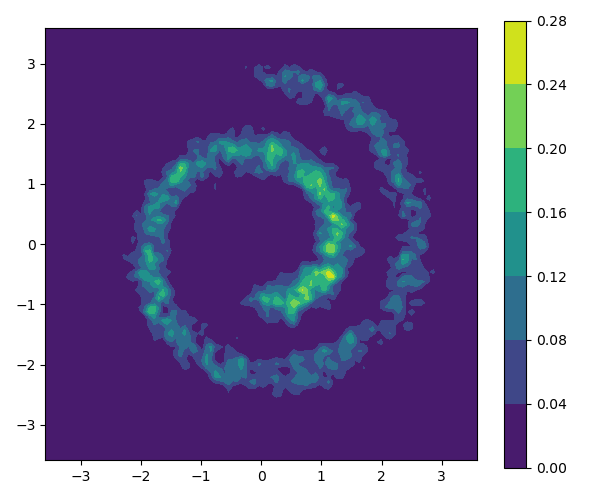}
  \end{minipage}
\begin{minipage}{0.1\textwidth}
$\epsilon= 0.1$
  \end{minipage}%
  \begin{minipage}{0.22\textwidth}
    \centering
    \includegraphics[width=\linewidth,trim=0 0 0 0, clip]{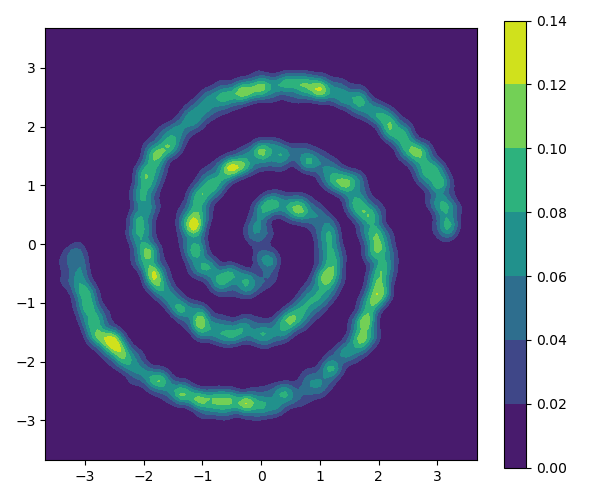}
  \end{minipage}%
  \begin{minipage}{0.22\textwidth}
    \centering
    \includegraphics[width=\linewidth,trim=0 0 0 0, clip]{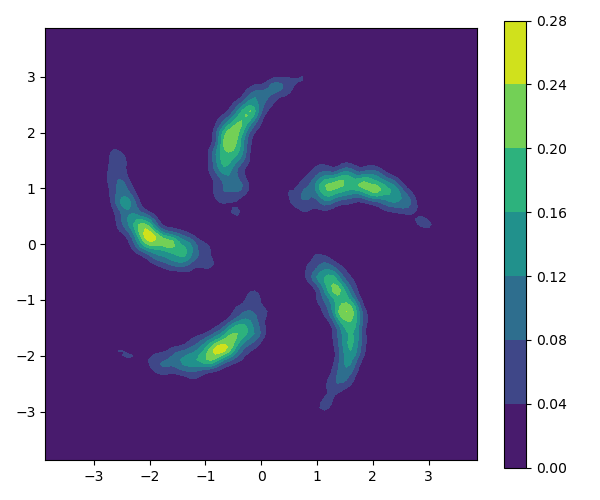}
  \end{minipage}%
  \begin{minipage}{0.22\textwidth}
    \centering
    \includegraphics[width=\linewidth,trim=0 0 0 0, clip]{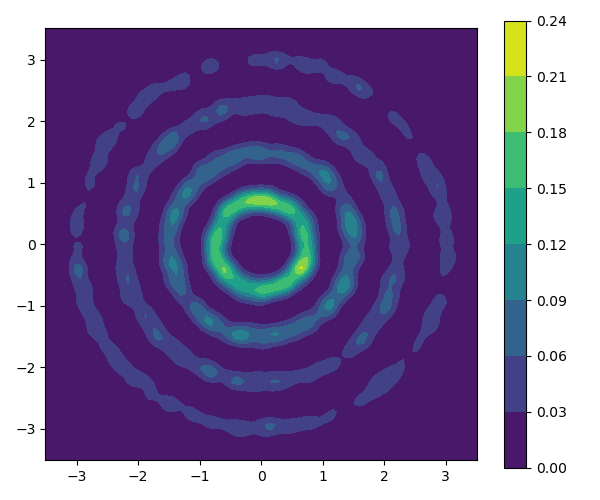}
  \end{minipage}%
  \begin{minipage}{0.22\textwidth}
    \centering
    \includegraphics[width=\linewidth,trim=0 0 0 0, clip]{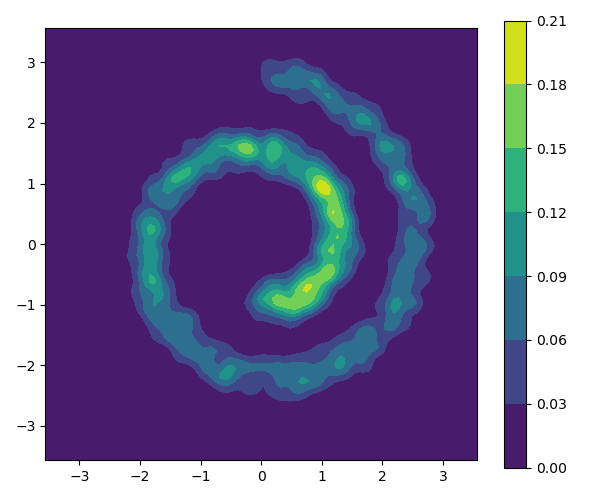}
  \end{minipage}

\begin{minipage}{0.1\textwidth}
$\epsilon= 0.2$
  \end{minipage}%
  \begin{minipage}{0.22\textwidth}
    \centering
    \includegraphics[width=\linewidth,trim=0 0 0 0, clip]{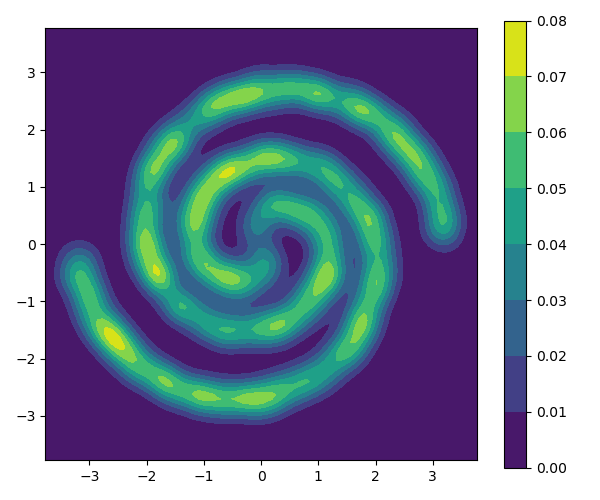}
  \end{minipage}%
  \begin{minipage}{0.22\textwidth}
    \centering
    \includegraphics[width=\linewidth,trim=0 0 0 0, clip]{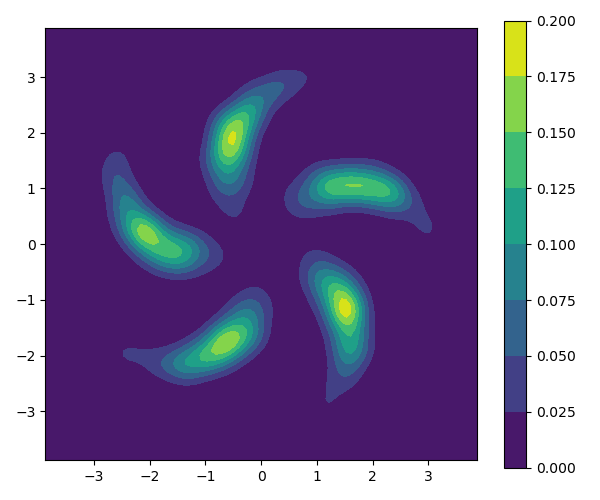}
  \end{minipage}%
  \begin{minipage}{0.22\textwidth}
    \centering
    \includegraphics[width=\linewidth,trim=0 0 0 0, clip]{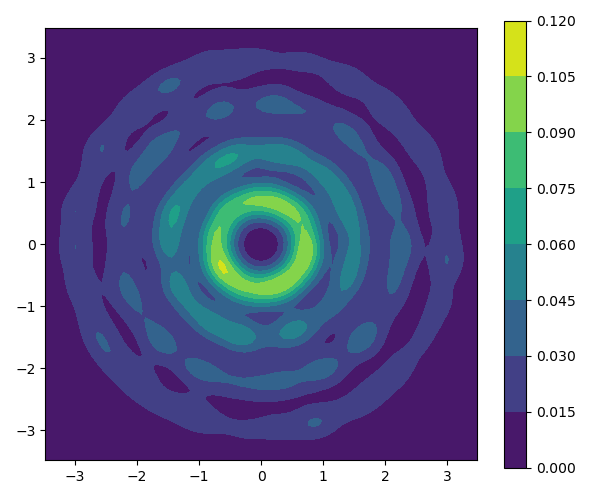}
  \end{minipage}%
  \begin{minipage}{0.22\textwidth}
    \centering
    \includegraphics[width=\linewidth,trim=0 0 0 0, clip]{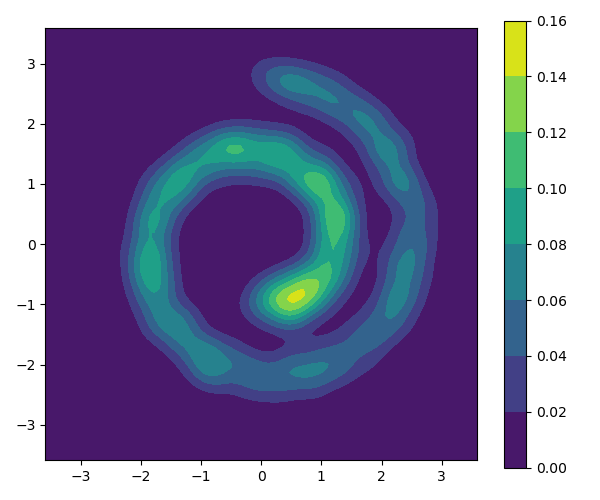}
  \end{minipage}

\begin{minipage}{0.1\textwidth}
WPO-informed kernel
  \end{minipage}%
  \begin{minipage}{0.22\textwidth}
    \centering
    \includegraphics[width=\linewidth,trim=0 0 0 0, clip]{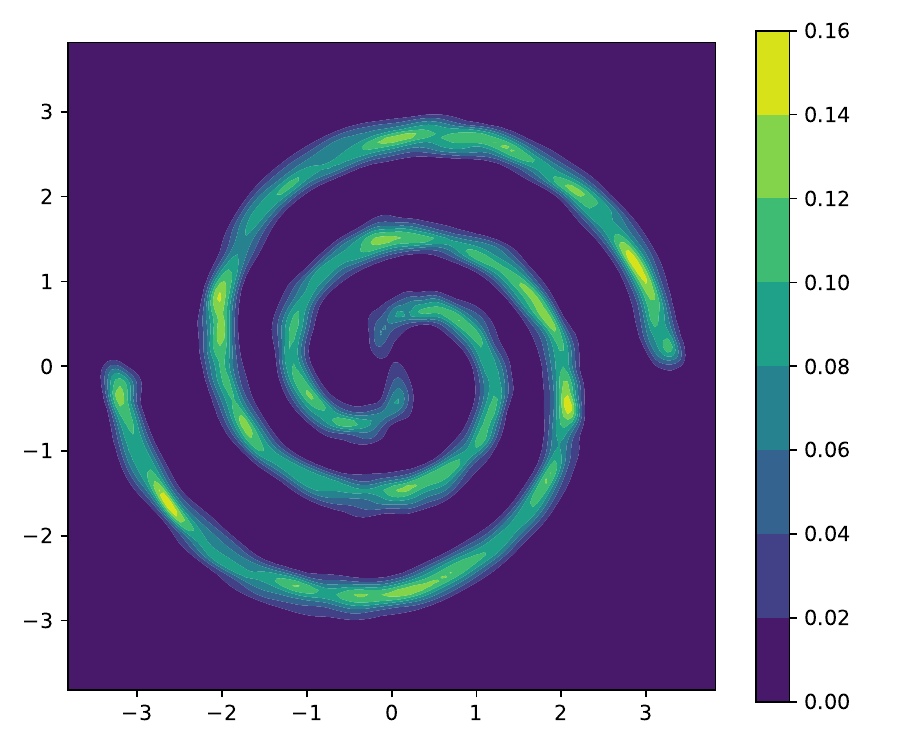}
  \end{minipage}%
  \begin{minipage}{0.22\textwidth}
    \centering
    \includegraphics[width=\linewidth,trim=0 0 0 0, clip]{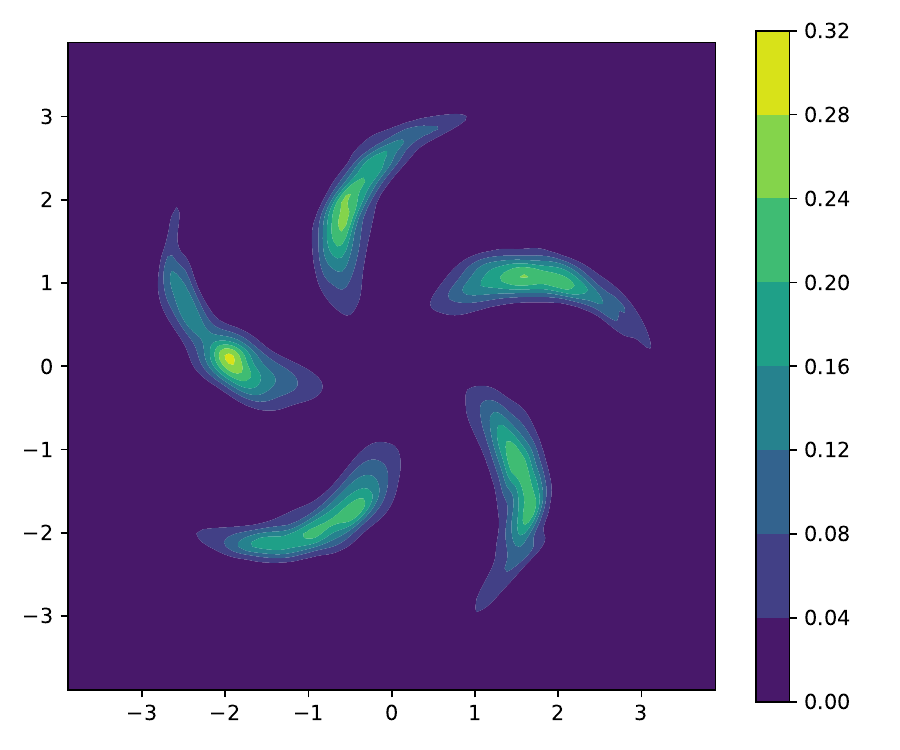}
  \end{minipage}%
  \begin{minipage}{0.22\textwidth}
    \centering
    \includegraphics[width=\linewidth,trim=0 0 0 0, clip]{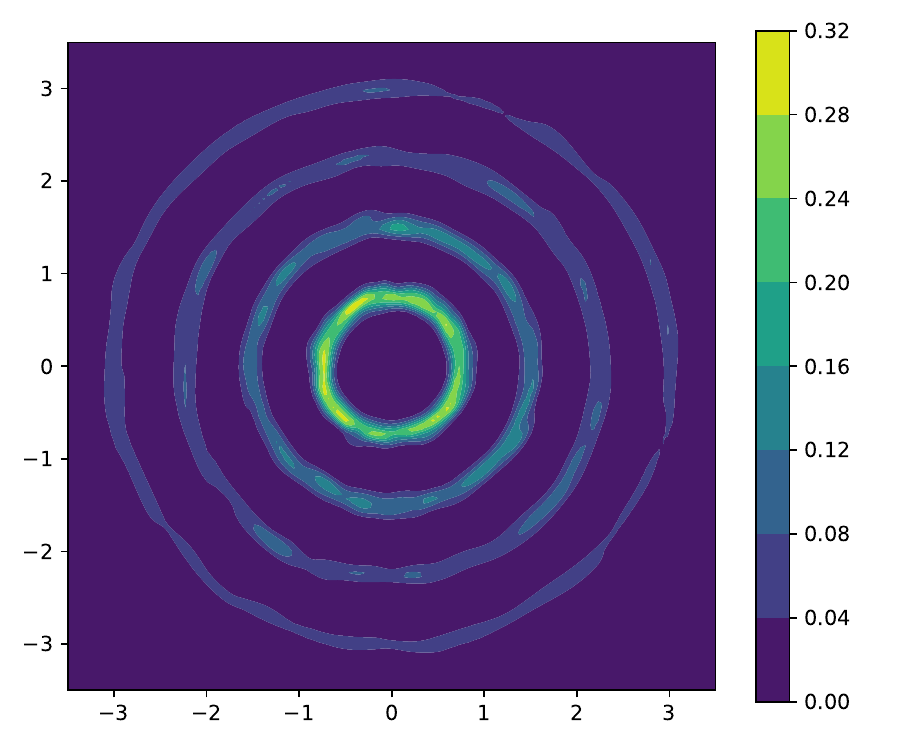}
  \end{minipage}%
  \begin{minipage}{0.22\textwidth}
    \centering
    \includegraphics[width=\linewidth,trim=0 0 0 0, clip]{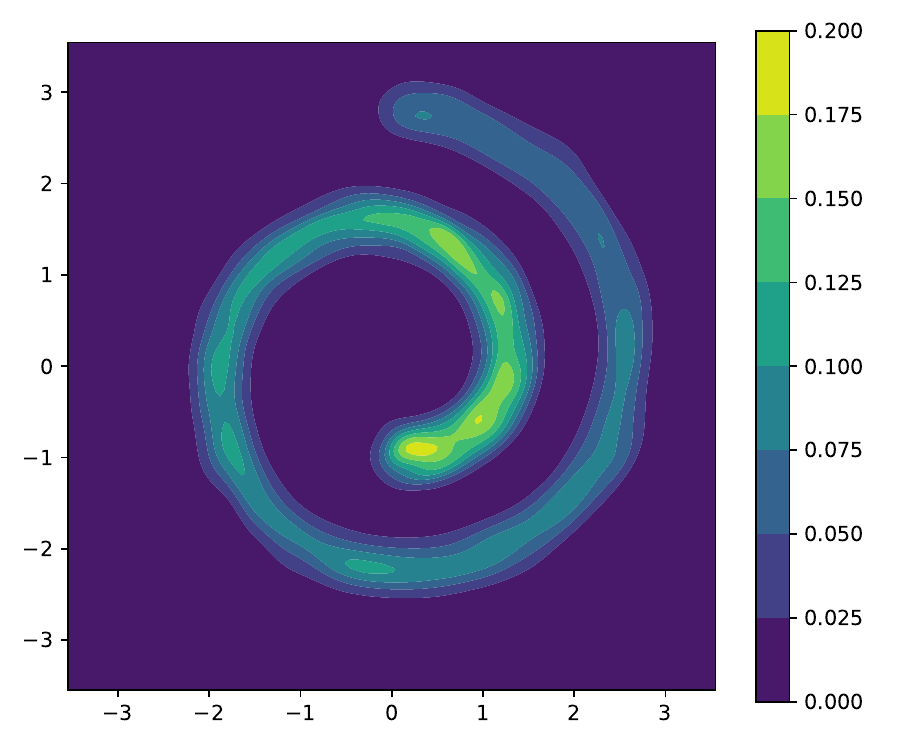}
  \end{minipage}
  
  \caption{
Comparison of density estimates between the empirical kernel model and WPO-informed kernels. All figures are estimated using a set of $2500$ samples for four datasets. In the first three rows, we employ an identity covariance with a constant scalar $\epsilon$. In the last row, we utilize the WPO-informed kernel model trained on $5000$ sample centers over $10^5$ epochs. Subsequently, plot the density using the trained model.
}
  \label{fig:wpo_vs_const_kernel}
\end{figure}

\section{Discussion: Resolving memorization and scalable implementations of the WPO-informed kernel model}

{
We presented a fundamental characterization of score-based generative models with SDEs in terms of the \textbf{Wasserstein proximal operator (WPO) of cross-entropy}. We showed that, through \textbf{mean-field games} and the \textbf{Cole-Hopf transformation}, the entropically regularized WPO precisely yields the canonical formulation of SGMs where the noising process is a Brownian motion. The optimality conditions of the mean-field game yield an \textbf{Hamilton-Jacobi-Bellman equation} that can be exactly expressed with a kernel representation formula \eqref{eq:repformula} (Proposition \ref{prop:repformula}).  Given finite training data, application of the kernel formula using the empirical distribution \eqref{eq:overfitsgm} for SGM will memorize the training data. We proposed a WPO-informed kernel model that is based on learning local precision matrices around some subset of the training data, which has the effect of \textbf{learning the manifold} of the underlying data distribution. The local precision matrices, and therefore the data manifold, are learned by fitting to the \textbf{terminal condition} of the HJB via implicit score-matching. We demonstrated through numerical examples that our approach results in a score model that is \emph{interpretable} and \textbf{learns faster with less data}. Here we summarize how the WPO-informed kernel model resolves memorization and how it may be scaled to higher dimensions via a deep learning connection.}

{\subsection{Resolving memorization in score-based generative models}
The memorization phenomenon of score-based generative models has been observed in simple toy models \cite{pidstrigach2022score,li2024good}, in small image datasets \cite{gu2023memorization}, large image datasets \cite{somepalli2023diffusion}, and in practical text-to-image applications \cite{somepalli2023understanding}. Moreover, \cite{pidstrigach2022score} studies the memorization phenomenon through SGM's manifold detecting behavior, showing that under suitable conditions, the optimal function that minimizes the DSM objective will result in a generative model that matches the support of the training data. In \cite{li2024good} the authors show that a neural net with sufficient complexity will, with enough training time, learn the empirical score kernel formula \eqref{eq:overfitsgm} very well. The learned score function yields a generative model whose generalization properties amount to early stopping of the denoising diffusion process, and that the produced samples come from a kernel density estimate. The score formula coming from the empirical distribution is typically cast in a negative light due its memorization effects. Our WPO-informed approach to SGM \emph{embraces} the kernel formula as kernels encode the core mathematical structure of score-based generative models, and we show that the kernel formula can be modified to avoid memorization. }

We summarize our contribution in explaining and resolving memorization effects in score-based generative modeling. The key reason for why SGMs memorize the training data is demonstrated by the optimal score function for the denoising score matching optimization problem \eqref{eq:dsm} which, assuming a finite sample, can be written in closed form using a kernel formula \eqref{eq:overfitsgm} based on the empirical distribution on the training data \eqref{eq:empiricaldistribution}. When the denoising diffusion process \eqref{eq:denoising} is simulated with the empirical score kernel formula \eqref{eq:overfitsgm}, the system will sample only from the the empirical distribution, thereby exhibiting memorization of the training data. 

{ Our WPO-informed kernel model for the score function relies on the connection between score-based generative models and the Wasserstein proximal operator \eqref{eq:wpo} via a mean-field game \eqref{eq:score_opt}. The form of the resulting MFG partial differential equations informs the Cole-Hopf and kernel structures of its solutions, yielding Proposition \ref{prop:repformula}, which provides a general kernel representation formula \eqref{eq:repformula} the score function. One approximation for \eqref{eq:repformula} is via the empirical distribution of the training samples \eqref{eq:overfitsgm}, but is not the only choice. The WPO-informed kernel formula \eqref{eq:kernelcovdensity} provides a \emph{smooth} approximation to the data distribution where local precision matrices around certain kernel centers are learned. The smoothness of the approximation is key to avoid memorization and improving generalization. The precision matrices are learned by only performing implicit score matching \emph{at the terminal time.} This is not possible with the denoising score matching objective as it will simply fit to the training data. The use of the implicit score matching objective is only possible with our model because it is \emph{smooth}, and by the fact in Proposition \ref{prop:hjbsolution}, shows that it is \textit{sufficient} to only learn the score at the terminal time as the kernel formula provides a closed-form solution for the score at all future times. In fact, we welcome a description of the WPO-informed kernel model as simply \emph{a Gaussian mixture model that is trained using the implicit score matching objective. } } 

\subsection{Scaling the WPO-informed kernel model for high dimensional applications: a bespoke neural network architecture} 

\label{sec:bespoke}

While the kernel-based score model yields an interpretable model that is constructed to respect its fundamental mathematical structure, they are generally not scalable due to the computational cost of evaluating the kernels. Here, we study the kernel-based model \eqref{eq:kernelcovdensity} and show that it admits a neural network interpretation, thereby providing means for scaling the kernel-based model to high-dimensional applications. In other words, we provide preliminary explorations of what are suitable neural network approximations to the kernel representation formula in Proposition \ref{prop:repformula}. {The resulting neural network architecture not only inherits the kernel structure from the WPO-informed model but, more fundamentally, the PDE structure informed by the Cole-Hopf formula. The resulting bespoke neural network then becomes yet another way to approximate the kernel representation formula in Proposition \ref{prop:repformula}. }

For $x\in \R^d$, let $(x\otimes x)$ be the column-wise Kronecker product
\begin{align}
    (x\otimes x) = \begin{bmatrix}
        x_1^2 & x_1x_2 & \cdots & x_1x_d & x_2^2 & \cdots & x_2x_d & \cdots & \cdots & x_d^2 
    \end{bmatrix}^\top \in \R^{d^2}.
\end{align}
 Define a lifting operator $\mathcal{T}: \R^d \to \R^{d^2+d}$, $\mathcal{T}(x) = \tilde{x} = \left[(x\otimes x)^\top,\, x^\top \right]^\top$. Recall the kernel score model \eqref{eq:scoremodel} and notice that it may be written in terms of the softmax function, which implies a neural network connection. Indeed, observe that 
 we may rewrite 
 the exponent in \eqref{eq:scoremodel} as a linear function in the lifted space, namely 
 \[ 
 \mathbf{y}_\theta(x) = \mathbf{A}_\theta\mathcal{T}(x) + \mathbf{b}_\theta\, ,
 \]
  where 
 $\mathbf{A}_\theta = \begin{bmatrix}
    \mathbf{G}_\theta & \mathbf{H}_\theta
\end{bmatrix}^\top \in \R^{N\times(d^2+d)}$, with matrices $\mathbf{G}_\theta \in \R^{N\times d^2}$, $\mathbf{H}_\theta \in \R^{N\times d}$, $\mathbf{b}_\theta \in \R^N$, and
\begin{align}
    &\mathbf{G}_\theta = \begin{bmatrix}
        \text{vec}(\bGamma_\theta(Z_1)) & \text{vec}(\bGamma_\theta(Z_2)) & \cdots & \text{vec}(\bGamma_\theta(Z_N)) 
    \end{bmatrix}^\top \\
       & \mathbf{H}_\theta = \begin{bmatrix}
        \bGamma_\theta(Z_1) Z_1 & \bGamma_\theta(Z_2) Z_2 & \cdots & \bGamma_\theta(Z_N) Z_N
    \end{bmatrix}^\top  \nonumber \\
    & (\mathbf{b}_\theta)_i = \log \det \bGamma_\theta(Z_i) - \frac{d}{2} \log 2\pi - \log N \nonumber.
\end{align}
Here, $\text{vec}:\R^{N\times d} \to \R^{Nd}$ is the vectorization operator that concatenates columns of a matrix into one vector. The kernel score formula \eqref{eq:scoremodel} is then written as 
\begin{align}\label{eq:bespoke_nn_score}
    \mathsf{s}_\theta(x) &= ( \nabla \mathbf{y}_\theta(x) )\sigma\left( \mathbf{A}_\theta \mathcal{T}(x) + \mathbf{b}_\theta\right), \\
    \nabla \mathbf{y}_\theta(x) &= \mathbf{F}_\theta x-\mathbf{H}_\theta^\top, \nonumber\\ 
    \mathbf{F}_\theta(x) &= \begin{bmatrix} \bGamma_\theta(Z_1) x & \bGamma_\theta(Z_2) x & \cdots & \bGamma_\theta(Z_N) x \end{bmatrix}. \nonumber
\end{align}

We see that the score kernel model can be interpreted in terms of a shallow neural network where the inputs are elements of the lifted space $\R^{(d^2+d)N}$. This is a specialized, \emph{bespoke} neural network architecture derived from the WPO-informed kernel formula \eqref{eq:kernelcovdensity} that we contend may learn score models in a scalable fashion in high-dimensional applications. For example, typical deep learning implementations of implicit score-matching requires the use of autodifferentiation packages to compute or approximate the divergence of the score model \cite{song2020sliced}. The specialized kernel formula will admit closed form expressions for the divergence via its connection to the kernel formulas noted in Remark \ref{remark:noautodiff}.  Moreover, the score functions in the family of neural networks with the specialized architecture \eqref{eq:bespoke_nn_score} at any time $t \in [0,T)$ can be derived via Proposition \ref{prop:hjbsolution}. 

We emphasize that there is a direct correspondence between learning the parameters of the neural network, i.e., $\mathbf{A}_\theta,\mathbf{b}_\theta$ to learning the local precision matrices, $\bGamma(Z_i)$, in the kernel density formula. In particular, notice that entries of $\mathbf{A}_\theta$ require knowledge of the local precision matrices around kernel centers as well as the centers $Z_i$ themselves, while learning $\mathbf{b}_\theta$ corresponds with a normalizing factor depending on the local precision matrices. The neural network interpretation suggests that the kernel centers can be \emph{learned} rather than placed \emph{a priori} as we did in Sections \ref{sec:wpokernelmodel} and \ref{sec:numerical}. These properties give further intuitive justification for how SGMs learn manifolds of the underlying data distribution --- when a generic neural net learns the score function, it is learning (1) the implicit mathematical structure that is inherent in the score function, i.e., the lifting transform $\mathcal{T}$, (2) the kernel centers, and (3) the local precision matrix around each kernel center. Furthermore, in contrast with typical neural network training procedures, we know the explicit relationships within parameters $\mathbf{A}_\theta$ and $\mathbf{b}_\theta$, meaning that their relation may be imposed within the training process and may accelerate their learning. In particular, the learning of the lifting operator $\mathcal{T}$ is, in effect, learning an explicit mathematical relationship that is invariant with respect to the training data. Learning this explicit relationship would introduce additional statistical errors than if the relation were included in the model \emph{a priori} to training.

\bibliographystyle{ieeetr}
\bibliography{bibliotheque}

\end{document}